%% file: main.tex
\documentclass{article}
\usepackage[left=2cm,right=2cm]{geometry}

\usepackage{graphicx}
\usepackage[utf8]{inputenc} 
\usepackage{authblk}

\usepackage[T1]{fontenc}    
\usepackage{hyperref}       
\usepackage{url}            
\usepackage{booktabs}       
\usepackage{microtype}      
\usepackage{xcolor}         
\usepackage{mathtools}
\usepackage{nicefrac}
\usepackage{multirow}
\usepackage{algorithmic}
\usepackage[ruled,vlined]{algorithm2e}
\usepackage{amsmath}
\usepackage{amssymb}
\usepackage{amsfonts}
\usepackage{amsthm}
\usepackage{wrapfig}
\usepackage{subcaption}
\newtheorem{theorem}{Theorem}
\newtheorem{lemma}{Lemma}

\newtheorem{corollary}{Corollary}

\newtheorem{remark}{Remark}
\usepackage{booktabs}
\usepackage{varwidth}
\newenvironment{cellvarwidth}
  {\begin{varwidth}{\textwidth}}  
  {\end{varwidth}}                
\title{TRUST: Test-time Resource Utilization for Superior Trustworthiness}
\date{}

\author[1]{Haripriya Harikumar \thanks{Currently with University of Manchester but the majority of the work was done at Applied Artificial Intelligence Institute, Deakin University, Geelong, Australia.}\thanks{Correspondence to haripriya.harikumar@manchester.ac.uk}}
\author[2]{Santu Rana}
\affil[1]{Department of Computer Science, University of Manchester, Manchester, UK}
\affil[2]{Applied Artificial Intelligence Institute, Deakin University, Geelong, Australia}
\begin{document}

\maketitle

\begin{abstract}
Standard uncertainty estimation techniques, such as dropout, often struggle to clearly distinguish reliable predictions from unreliable ones. We attribute this limitation to noisy classifier weights, which, while not impairing overall class-level predictions, render finer-level statistics less informative. To address this, we propose a novel test-time optimization method that accounts for the impact of such noise to produce more reliable confidence estimates. This score defines a monotonic subset-selection function, where population accuracy consistently increases as samples with lower scores are removed, and it demonstrates superior performance in standard risk-based metrics such as AUSE and AURC. Additionally, our method effectively identifies discrepancies between training and test distributions, reliably differentiates in-distribution from out-of-distribution samples, and elucidates key differences between CNN and ViT classifiers across various vision datasets.
\end{abstract}

\input{sec/1_intro}

\input{sec/2_related}

\input{sec/3_method}

\input{sec/4_expmnts}

\input{sec/6_conclusion.tex}

\bibliographystyle{plain}
\bibliography{ref}

\newpage
\begin{center}
    \LARGE \textbf{Supplementary Material}
\end{center}

\input{sec/suppl}

\end{document}

%% file: sec/1_intro.tex
\section{Introduction}

Deep learning-based vision classifiers have shown remarkable performance
across various domains \cite{zhou2023deep,zhu2023unsupervised,gheisari2023deep}.
However, even highly accurate models can make inexplicable errors
and exhibit unwarranted confidence when confronted with data that
diverges from the training distribution. Such behavior is unacceptable
in high-stakes settings, such as in diagnostic systems based on
medical image classification \cite{flosdorf2024skin}. Accurately
estimating prediction confidence scores is essential to identifying
when a model might make mistakes, enabling actions like rejecting the
prediction or involving experts in human-machine teaming \cite{mosqueira2023human}. 

Confidence estimation in deep learning models is commonly achieved
through methods like softmax probabilities \cite{pearce2021understanding},
Monte Carlo dropout \cite{kingma2015variational,gal2016dropout},
ensemble methods \cite{lakshminarayanan2017simple}, and Bayesian
neural networks \cite{kononenko1989bayesian}. These techniques attempt
to quantify uncertainty by repurposing model outputs or generating
multiple predictions. However, except for basic softmax-based measures
\cite{pearce2021understanding}, most of these approaches require
specialized training techniques, and methods like Bayesian neural
networks are often impractical for large-scale applications due to
their high computational demands. Additionally, all of them struggle
in providing reliable confidence estimates despite correctly predicting
class labels. We hypothesize that this is due to their assumption
of an ideal, perfect model, a requirement seldom met in real-world
scenarios. In most practical training runs, the training loss may
not fully converge, leading to noisy model weights. Even with zero
training loss, overparameterized models may retain noisy parameters
that become activated with unseen data during testing, adversely affecting
the confidence estimation process. LogitNorm \cite{wei2022mitigating},
a more recent approach to control the unnecessary overconfidence by
limiting the norm of the logit output, is a step in the right direction.
However, it requires re-training based on the proposed loss function
and it also lacks any theoretical underpinning.

To address this issue, we first show how noise in classifier weights distorts confidence estimation, and then propose a sample-specific, test-time optimization strategy to reduce this noise. In high-dimensional feature spaces, training data tend to lie on a hypersphere and form microclusters, each centered around a mode and well separated from others. This aligns with the insights from Neural Collapse \cite{pan2023towards}, which shows that final-layer representations often collapse to a single point per class. We generalize this observation by allowing classes to collapse across multiple modes. Under this view, the angular distance between a test point and its nearest mode serves as an effective proxy for epistemic uncertainty. We estimate this mode through a lightweight, sample-specific optimization procedure and compute uncertainty as the angular distance to the estimated mode. While prior work such as \cite{jiang2018trust} has shown that distance-based measures can better capture epistemic uncertainty than traditional methods, their approach depends on identifying high-density regions in the training data, making it sensitive to out-of-distribution samples and blind to the effect of weight noise. 

While the noise inherent in a classifier’s weights cannot be effectively made zero, especially in overparameterized settings, we argue that it is possible to mitigate its effect on mode estimation by leveraging additional computation at test time. To demonstrate this, we introduce TRUST (Test-time Resource Utilization for Superior Trustworthiness), a novel reliability score, which simply measures the cosine distance between the test sample and its nearest mode. We observe that TRUST defines a monotonic set function over the test population: as samples are filtered based on higher TRUST scores, overall accuracy consistently improves. Moreover, TRUST outperforms conventional uncertainty quantification techniques on risk-based metrics such as AUSE and AURC. 

Importantly, we also observe that the distributional gap between training and test TRUST scores provides early signals about generalization performance, offering a new lens through which to evaluate a classifier’s suitability for deployment.  In summary, our main contributions are as follows:

\textbf{Analysis of Noisy Model Weights: }We analyze the impact of noisy model weights on confidence estimation, demonstrating why current
confidence scoring methods are often unreliable.

\textbf{Novel Test-Time Approach for Confidence Estimation: }We introduce
a first-of-its-kind approach that leverages test-time computation
to mitigate the effects of noisy weights, improving the reliability
of confidence scores.

\textbf{Introduction of TRUST Score: }We propose a new metric, the
TRUST score, which achieves state-of-the-art performance in identifying
reliable predictions and shows potential in other valuable use cases,
such as detecting out-of-distribution samples, predicting performance
on non-aligned test distributions, and revealing insights into model
behavior.

We conduct extensive analysis in a range of four benchmark datasets
(CIFAR-10, CAMELYON-17, TinyImagenet, and Imagenet) and models, ranging
from simple to state-of-the-art ViT models, demonstrating
the broad applicability and robustness of our approach. \textcolor{blue}{Code is available at \href{https://shorturl.at/Ka77k}{LINK}}.

%% file: sec/2_related.tex
\section{Related work}

Traditional methods like Bayesian neural networks \cite{kononenko1989bayesian},
dropout-based variational inference \cite{kingma2015variational,gal2016dropout}
focus on epistemic uncertainty \cite{swiler2009epistemic} but are
computationally intensive. Recent methods, including Dirichlet-based
and evidential learning models separate aleatoric \cite{kendall2017uncertainties}
and epistemic uncertainty but still face challenges with noise and
dataset shifts \cite{deng2023uncertainty,mukhoti2023deep,hullermeier2021aleatoric}.
A non-Bayesian approach as proposed in \cite{lakshminarayanan2017simple} use ensemble models, however they rely on pre-trained model variance and adversarial
robustness. Efficient approaches like Deep Deterministic Uncertainty
\cite{mukhoti2023deep} employ single-pass networks with regularized
feature spaces, enabling useful uncertainty estimation in large-scale
applications \cite{chua2023tackling}. A recent work \cite{huang2023look}
provides insights into uncertainty quantification challenges and techniques
for high-dimensional language models, relevant for understanding scalability
and calibration in large-scale deep learning models across domains.

Emerging methods like Density-Aware Evidential Deep Learning \cite{yoon2024uncertainty}
and Fisher Information-based Evidential Learning \cite{deng2023uncertainty}
enhance Out-Of-Distribution (OOD) detection and few-shot performance
by integrating feature-space density and adaptive uncertainty weighting,
respectively, offering resilience under varied data conditions \cite{qu2022improving,hullermeier2021aleatoric}.
RCL \cite{zhu2024rcl} employs a continual learning paradigm for unified
failure detection, while LogitNorm \cite{wei2022mitigating} mitigates
overconfidence by constraining logit magnitudes during training. SIRC
\cite{xia2022augmenting} augments softmax scores for selective classification,
and OpenMix \cite{zhu2023openmix} utilizes outlier transformations
to improve misclassification detection. Confidence calibration methods like \cite{zhu2022rethinking} explore flat minima to enhance failure prediction. Unlike
LogitNorm \cite{wei2022mitigating}, TRUST goes further in employing
test-time computation to identify the effect of noisy weights. TRUST
quantifies test-time epistemic uncertainty via feature-space distances
without altering training or requiring extra data. 

Robust uncertainty estimation under dataset shifts is essential as
traditional calibration methods often degrade in non-i.i.d. conditions
\cite{ovadia2019can}. ODIN \cite{liang2017enhancing} and Generalized-ODIN
\cite{hsu2020generalized} improves OOD detection through input pre-processing
and temperature scaling but requires specific OOD tuning, limiting
flexibility \cite{shafaei2018less}. A recent work in \cite{van2024can},
explored the use of synthetic test data to better evaluate model performance
under shifts by simulating diverse scenarios, enhancing subgroup and
shift evaluation where real data may be limited. Thus OOD identification
is still a major unsolved problem \cite{gawlikowski2023survey}. While
human-in-the-loop \cite{mosqueira2023human,shergadwala2022human}
frameworks improve reliability in settings like healthcare \cite{bakken2023ai,zhou2023deep}
or finance \cite{heaton2017deep}, real-time human input is often
impractical and thus they need to be called only when it's absolutely
necessary. 

%% file: sec/3_method.tex
\section{Proposed approach}

We denote the training data distribution by $P_{X\times Y}$ and the
test data distribution by $Q_{X\times Y}$. We refer to the neural
network trained on $P_{X\times Y}$ as $f_{\theta}:\mathbb{R}^{\text{dim}(X)}\rightarrow[0\;1]^{\text{dim}(Y)}$,
where $\mathbb{R}$ represents the real-number line and $\theta$
being the trainable weights. For a sufficiently large model relative
to the dataset complexity, which is generally the case for modern,
overparameterized deep models, we can expect nearly all training data
to be classified correctly with near-perfect confidence.
\begin{wrapfigure}{R}{0.6\columnwidth}%
\begin{centering}
\includegraphics[width=0.3\columnwidth]{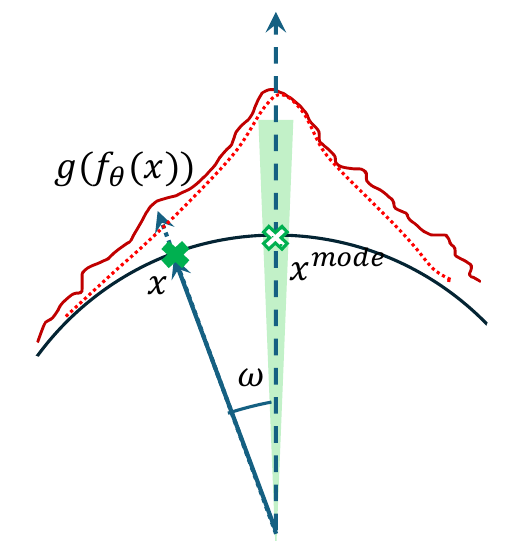}
\par\end{centering}
    \caption{Geometrical Intuition of Our Approach: Traditional methods rely on
the noisy approximation (solid red line) of the ground truth score
function (dotted red line) for uncertainty quantification. In contrast,
our approach directly computes the angular distance ($\cos(\omega$)
) between $x$ and its nearest mode $x^{\text{mode}}$. The green
shaded region illustrates the estimation uncertainty of $x^{\text{mode}}$,
which diminishes with additional test-time computation, enabling highly
precise estimation of $\cos(\omega)$. }\label{fig:Geometrical-intuition-of}
\vspace{-10bp}
\end{wrapfigure} 

In such a scenario, considering the extremely high dimensionality of the feature space induced by deep models, we can safely assume that the data are
spread over the surface of a hypersphere, with islands of micro-clusters dominating the landscape. The larger the model, the smaller the size
of these micro-clusters, and the further apart they become from each
other. Thus, it is not surprising that extremely large models often
show extreme memorization ability \cite{carlini2022quantifying,biderman2024emergent},
and can behave like nearest-neighbor classifiers in the feature space.
Under the assumption of such a topology, where data appear in micro-clusters
that are relatively far from each other, which is more pronounced
in higher-dimensional feature spaces, we can deduce that the distance
of a test data point from the median or mode of its nearest micro-cluster
should be proportional to the classifier's epistemic uncertainty.

Traditional approaches to estimating confidence typically incorporate
an auxiliary function $g:\mathbb{R}^{\text{dim}(Y)^{k}}\rightarrow\mathbb{R}_{\geq0}$,
which uses $k$ outcomes or inferences from $f_{\theta}$ (e.g., using
MC-dropout-based measures) to estimate angular distance or a monotonic
transformation of it. However, because $g$ relies on the outputs
of $f_{\theta}$, it remains vulnerable to noise in the model weights
$\theta$. In the following section, we propose a direct method for
measuring angular distance that bypasses this noise, resulting in
a more accurate estimate of the model’s confidence in a given prediction
(in Fig \ref{fig:Geometrical-intuition-of}).

\subsection{TRUST: Test-time Resource Utilization for Superior Trustworthiness}

In this section, we develop the core methodology for computing the
angular distance between a test sample $x_{\text{test}}\in Q$ and
its nearest mode in the feature space. Identifying the exact mode
of the micro-cluster that $x_{\text{test}}$ belongs to within the
training data is challenging, as it would require mapping out the
entire data manifold, a requirement that is infeasible with large-scale
datasets. Instead, we propose a test-time optimization approach that
projects $x_{\text{test}}$ to its nearest mode (i.e., maximizing
the probability of the predicted class to 1) by introducing only minimal
changes.

This optimization transforms $x_{\text{test}}$ into its nearest cluster
mode $x_{\text{test}}^{\text{mode}}=x_{\text{test}}+\Delta x$ by
solving the following problem:
\vspace{-3bp}
\begin{equation}
\arg\min_{\Delta x}\,\,L(x_{\text{test}}+\Delta x,y_{\text{test}})+\lambda||\Delta x||_{1}\label{eq:optimization}
\end{equation}

where $y_{\text{test}}$ is the predicted class label for $x_{\text{test}}$
by the classifier $f_{\theta}$, and $L(.)$ is the loss function,
typically cross-entropy for classification. The $L_{1}$-norm regularization
term ensures that $\Delta x$ remains sparse, so $x_{\text{test}}$
undergoes minimal modification, reducing the risk of being assigned
to a different micro-cluster during optimization. 
The weight $\lambda$ is usually set to a high value to ensure a sparse solution, and a long optimization should ensure this loss function to reaching very close to optima.

To refine the estimation of $\Delta x$, we increase the softmax temperature
$T$ when computing the softmax score for class $i$:
\vspace{-8bp}
\[
S_{i}(x,T)=\frac{\exp(f_{\theta,i}(x)/T)}{\sum_{j=1}^{k}\exp(f_{\theta,j}(x)/T)}
\]

This temperature amplifies the differences between predictions for
$x_{\text{test}}$ and $x_{\text{test}}^{\text{mode}}$, enhancing
the ability of optimization to fine-tune $\Delta x$ , $k$ is the
number of classes. The TRUST score is then computed as the cosine distance between $x_{\text{test}}$
and $x_{\text{test}}^{\text{mode}}$ in the feature space. Although
deep models offer various feature spaces (e.g., layer-wise or combined
layers), our experiments indicate final layer to be the most reliable.
Algorithm 1 presents the TRUST score computation, and the convergence criterion in Step 3 is based on loss convergence.
\begin{algorithm}
\caption{TRUST Score computation}
\begin{algorithmic}[1]
\STATE \textbf{Input}: Test sample $x_{\text{test}} \in Q$, classifier $f_{\theta}$, target class label $y_{\text{test}}=\arg\max_{y \in \{1,\dots,k\}} f_\theta(x_{\text{test}})$, regularization weight $\lambda$, softmax temperature $T$
\STATE \textbf{Initialize}: Set initial perturbation $\Delta x = 0$
\WHILE{loss convergence criterion is not met}
    \STATE Compute softmax score: $S_{y_{\text{test}}}(x_{\text{test}} + \Delta x, T) = \frac{\exp(f_{\theta,y_{\text{test}}}(x_{\text{test}} + \Delta x) / T)}{\sum_{j=1}^{k} \exp(f_{\theta,j}(x_{\text{test}} + \Delta x) / T)}$
    \STATE Define objective: $\mathcal{L} = L(x_{\text{test}} + \Delta x, y_{\text{test}}) + \lambda \|\Delta x\|_{1}$
    \STATE Update $\Delta x$ by minimizing $\mathcal{L}$
\ENDWHILE
\STATE \textbf{Obtain Nearest Mode:} $x_{\text{test}}^{\text{mode}} = x_{\text{test}} + \Delta x$
\STATE \textbf{Compute TRUST Score:} 
$\text{TRUST}(x_{\text{test}}) = \frac{ f_{\theta}^{l}(x_{\text{test}})^\top f_{\theta}^{l}(x_{\text{test}}^{\text{mode}}) }{ \| f_{\theta}^{l}(x_{\text{test}}) \|_{2} \cdot \| f_{\theta}^{l}(x_{\text{test}}^{\text{mode}}) \|_{2} }$
\STATE \textbf{Output}: TRUST Score for $x_{\text{test}}$
\end{algorithmic}
\end{algorithm}

\subsection{Mathematical analysis }

In this section first we revisit some of the well-known results related
to high-dimensional geometry that served as the primary motivation
of our work. Next, we deduce the effect of noisy weights on the classifier
derived confidence score and show how our method offers superior robustness.
\begin{theorem}
(\textbf{Concentration of Measure on the Hypersphere}):
\end{theorem}

Let $x$ be a random vector in $\mathbb{R}^{d}$, where each component
$x_{i}$ is independently drawn from a distribution with mean 0 and
variance $\sigma^{2}$. Define the Euclidean norm of $x$ as $||x||=\sqrt{\sum_{i=1}^{d}x_{i}^{2}}$.
\begin{enumerate}
\item \textbf{Norm Concentration: }As the dimensionality $d\to\infty$ ,
the Euclidean norm $||x||$ concentrates around $\sqrt{d}\sigma$
, meaning that for any small $\epsilon>0$ , $Pr\left(\left|||x||-\sqrt{d}\sigma\right|<\epsilon\sqrt{d}\sigma\right)\to1$.
\item \textbf{Surface Concentration:} Consequently, as $d$ grows large,
the points $x$ become increasingly concentrated near the surface
of a hypersphere with radius $\sqrt{d}\sigma$ centered at the origin.
Specifically, the probability that a randomly chosen point lies within
a thin shell of radius $\sqrt{d}\sigma\pm\epsilon$ approaches 1 as
$d\to\infty$ .
\end{enumerate}
\begin{proof}
Well known.
\end{proof}
This well-known theorem demonstrates that in high-dimensional spaces,
data tends to concentrate near the surface of a hypersphere.
\begin{theorem}
Let n points be independently and uniformly distributed on the surface
of a $d$-dimensional unit hypersphere. Then, the expected minimum
value of $\cos(\omega)$ , where $\omega$ is the angle between any
pair of points, is approximately given by $E[\cos(\omega_{\text{min}})]\approx-\sqrt{\frac{2\ln n}{d}}$, where $\cos(\omega_{\text{min}})$ denotes the minimum cosine value
over all pairs of points under the assumption of both $n$ and $d$
being large.
\end{theorem}

\begin{proof}
The proof can be obtained by first observing that $cos(\omega)$ has
an approximate distribution of $\mathcal{N}(0,\frac{1}{d})$ (see
the corresponding Lemma in the supplementary) and then using the Extreme
Value Theory for $n\to\infty$ we can prove the result (see supplementary
for details).
\end{proof}
\begin{corollary}
In an extreme high-dimensional feature space any two points are nearly-orthogonal.
\end{corollary}

\begin{proof}
As per the previous theorem as $d\to\infty$ $E[\cos(\omega_{\text{min}})]\to0$.
\end{proof}
The corollary above explains why we expect micro-clusters of training
data in the feature spaces of large, deep models to be well-separated.
This separation enables us to focus primarily on the nearest micro-cluster
to a test data point, which then predominantly influences uncertainty
quantification.

\subsubsection{Error analysis }

In this section first we derive the error probability when a noisy
scoring function is used to sort a list of items (i.e., stratification
in our case) then we show that our method provide exceptional noise
robustness.
\begin{lemma}
Let $x_{1},x_{2},\ldots,x_{n}$ be a set of $n$ items, each associated
with a true score $s_{i}$ for $i=1,2,\dots,n$ . Suppose the observed
score $\tilde{s_{i}}$ of each item $x_{i}$ is corrupted by Gaussian
noise $\varepsilon_{i}$ with mean zero and standard deviation $\sigma$
, such that:
\vspace{-3bp}
\[
\tilde{s_{i}}=s_{i}+\varepsilon_{i},\quad\text{where }\varepsilon_{i}\sim\mathcal{N}(0,\sigma^{2})
\]

Let $\Delta s_{ij}=s_{i}-s_{j}$ represent the true difference in
scores between items $i$ and $j$ , and let $\sigma_{ij}^{2}=2\sigma^{2}$
denote the variance of the difference $\Delta\tilde{s_{ij}}=\tilde{s_{i}}-\tilde{s_{j}}$
due to noise. Define a sorting error to occur if the observed ordering
based on $s_{i}{\prime}$ differs from the true ordering based on
$s_{i}$. Then, the probability $P_{ij}$ of a sorting error between
items $i$ and $j$ (i.e., $\tilde{s_{i}}<\tilde{s_{j}}$ when $s_{i}>s_{j}$)
is given by:
\vspace{-3bp}
\[
P_{ij}=P(\tilde{s_{i}}<\tilde{s_{j}})=1-\Phi\left(\frac{\Delta s_{ij}}{\sqrt{2}\sigma}\right)
\]
where $\Phi$ is the cumulative distribution function of the standard
normal distribution.
\end{lemma}

\begin{proof}
Straightforward and provided in supplementary.
\end{proof}
Next, we show that when we use cosine distance between $x_{\text{test}}$
and its estimated nearest mode $x_{\text{test}}^{\text{mode}}$ then
the probability of making sorting error due to same level of Gaussian
noise in the estimation is upper bounded by the probability of making
sorting error when the score function has the same level of noise.
\begin{theorem}
Let s$=\cos(\omega)$ be the cosine similarity between a test data
point $x_{\text{test}}$ and its nearest mode $x_{\text{test}}^{\text{mode}}$
, with angle $\omega$ between them. Suppose Gaussian noise $\delta\theta\sim\mathcal{N}(0,\sigma_{\omega}^{2})$
is added to $\omega$, resulting in a noisy score $s{\prime}=\cos(\omega+\delta\omega)$.
For an equivalent score function $s$ with direct Gaussian noise $\delta\omega\sim\mathcal{N}(0,\sigma_{\omega}^{2})$
, the probability $P_{\text{cos}}$ of a sorting error in the cosine
distance scenario is upper-bounded by the probability $P_{\text{direct}}$
of a sorting error in the direct noise scenario:
\vspace{-3bp}
\[
P_{\text{cos}}\leq P_{\text{direct}}
\]
where the bound holds because $\sigma_{s}^{2}=\sin^{2}(\omega)\cdot\sigma_{\omega}^{2}\leq\sigma_{\omega}^{2}$,
with equality when $\omega=\frac{\pi}{2}$ .
\end{theorem}

\begin{proof}
Straightforward using Taylor expansion of $cos(\omega+\delta\omega)$
and provided in supplementary.
\end{proof}
\vspace{-3bp}
\begin{remark}
For small values of $\omega$, where a test data point is classified
with high confidence due to its proximity to a micro-cluster, we observe
that $\sigma_{\omega}$ can become significantly larger than $\sigma_{s}$
to match the sorting error probability. In other words, by sufficiently
reducing $\sigma_{\omega}$ through adequate computation to optimize
Eq. \ref{eq:optimization}, achieving a matching sorting error probability
would require an extremely low noise level in the score function,
on the order of $\frac{1}{100}$ for $\omega=5^{\circ}$. As such
precision in $f_{\theta}$ is generally difficult to achieve, this
implies that, in most cases, our method should yield a markedly more
accurate uncertainty measure than traditional methods.
\end{remark}

%% file: sec/4_expmnts.tex
\section{Experimental setup}

\subsection{Datasets}
\textbf{CIFAR-10} \cite{Krizhevsky_etal_14Cifar} is a $32\times32\times3$
color image dataset with 10 classes.
\textbf{CAMELYON-17} \cite{camelyon} is a medical imaging dataset
containing images of size $96\times96\times3$ and with binary labels
of malignant or benign. It comprises of patches extracted from 50
Whole-Slide Images (WSI) of breast cancer metastases in lymph node
sections, with 10 WSIs from each of 5 hospitals in the Netherlands.
The training set has 302,436 patches from 3 hospitals, the validation
set 34,904 from a 4th, and the test set 85,054 from a 5th hospital.
\textbf{TinyImagenet} is a $64\times64\times3$ color image dataset
with 100,000 training samples across 200 classes.
\textbf{Imagenet }\cite{deng2009imagenet,ILSVRC15} is a $224\times224\times3$
color image dataset with 100,000 training samples across 1,000 classes.
\textbf{Noisy data}: CIFAR-10 test data with various noise distortions such as Uniform, Gaussian noise and brightness levels. \textbf{SVHN as OOD} \cite{netzer2011reading} is house numbers in
Google Street View images with 10 classes. We used 26,032 test images
of size $32\times32\times3$ as an Out-Of-Distribution (OOD) dataset
for the CIFAR-10 model. 

\begin{table*}
\begin{centering}
{\scriptsize{}%
\begin{tabular}{p{1.75cm}cccccccc}
\toprule 
\multirow{2}{*}{{\scriptsize Dataset}} & \multirow{2}{*}{{\scriptsize Method}} & \multicolumn{5}{c}{{\scriptsize Accuracy @ Top-\% Data (by TRUST Score) $\uparrow$}} & \multirow{2}{*}{{\scriptsize AURC$\downarrow$}} & \multirow{2}{*}{AUSE{\scriptsize$\downarrow$}}\tabularnewline
\cmidrule{3-7}
 &  & {\scriptsize 20} & {\scriptsize 40} & {\scriptsize 60} & {\scriptsize 80} & {\scriptsize 100} &  & \tabularnewline
\midrule
\midrule 
\multirow{11}{*}{\begin{cellvarwidth}
\centering
{\scriptsize CIFAR-10}{\scriptsize\par}
\end{cellvarwidth}} 
& {\scriptsize Dropout} & {\scriptsize 98.15} & {\scriptsize 98.08} & {\scriptsize 97.95} & {\scriptsize 97.88} & {\scriptsize 89.0} & 0.024 & 0.019\tabularnewline
\cmidrule{2-9}
& {\scriptsize Density aware} & {\scriptsize 99.95} & {\scriptsize 99.83} & {\scriptsize 99.33} & {\scriptsize 96.64} & {\scriptsize 86.61} & 0.0196 & 0.010\tabularnewline
\cmidrule{2-9}
& {\scriptsize ViM} & {\scriptsize 92.65} & {\scriptsize 92.40} & {\scriptsize 92.38} & {\scriptsize 92.25} & {\scriptsize 92.06} & 0.076 & 0.073\tabularnewline
\cmidrule{2-9}
& {\scriptsize SIRC (MSP,\textbar\textbar z\_1)} & {\scriptsize 91.30} & {\scriptsize 92.30} & {\scriptsize 92.47} & {\scriptsize 92.15} & {\scriptsize 92.06} & 0.083 & 0.079\tabularnewline
\cmidrule{2-9}
& {\scriptsize SIRC (MSP,\textbar\textbar res)} & {\scriptsize 92.65} & {\scriptsize 92.40} & {\scriptsize 92.38} & {\scriptsize 92.25} & {\scriptsize 92.06} & 0.076 & 0.073\tabularnewline
\cmidrule{2-9}
& {\scriptsize SIRC (-H,\textbar\textbar z\_1)} & {\scriptsize 91.30} & {\scriptsize 92.30} & {\scriptsize 92.47} & {\scriptsize 92.15} & {\scriptsize 92.06} & 0.083 & 0.079\tabularnewline
\cmidrule{2-9}
& {\scriptsize SIRC (-H,\textbar\textbar Res)} & {\scriptsize 92.65} & {\scriptsize 92.40} & {\scriptsize 92.26} & {\scriptsize 92.27} & {\scriptsize 92.06} & 0.076 & 0.073\tabularnewline
\cmidrule{2-9}
& {\scriptsize LogitNorm} & {\scriptsize \bf{99.95}} & {\scriptsize \bf{99.75}} & {\scriptsize \bf{99.38}} & {\scriptsize \bf{98.53}} & {\scriptsize \bf{94.36}} & 0.009 & 0.007\tabularnewline
\cmidrule{2-9}
& {\scriptsize LogitNorm+TRUST} & {\scriptsize \bf 99.95} & {\scriptsize \bf 99.95} & {\scriptsize \bf 99.82} & {\scriptsize \bf 99.21} & {\scriptsize \bf 94.36} & \textbf{0.006} & \textbf{0.005}\tabularnewline
\cmidrule{2-9}
& {\scriptsize CrossEntro+TRUST} & {\scriptsize \bf{100.0}} & {\scriptsize \bf{99.98}} & {\scriptsize \bf{99.65}} & {\scriptsize \bf{98.19}} & {\scriptsize \bf{92.06}} & 0.011 & 0.007\tabularnewline
\midrule 
\multirow{2}{*}{\begin{cellvarwidth}
\centering
{\scriptsize CAMELYON-17}{\scriptsize\par}
\end{cellvarwidth}} 
& {\scriptsize Dropout} & {\scriptsize 87.03} & {\scriptsize 81.93} & {\scriptsize 85.61} & {\scriptsize 87.77} & {\scriptsize 85.24} & 0.14 & 0.164\tabularnewline
\cmidrule{2-9}
& {\scriptsize CrossEntro+TRUST } & {\scriptsize \bf 93.46} & {\scriptsize \bf 90.41} & {\scriptsize \bf 88.25} & {\scriptsize \bf 85.86} & {\scriptsize \bf 83.52} & \textbf{0.10} & \textbf{0.089}\tabularnewline
\midrule 
\multirow{2}{*}{\begin{cellvarwidth}
\centering
{\scriptsize TinyImagenet}{\scriptsize\par}
\end{cellvarwidth}} 
& {\scriptsize Dropout} & {\scriptsize 88.57} & {\scriptsize 92.88} & {\scriptsize 94.29} & {\scriptsize 94.29} & {\scriptsize 81.71} & 0.09 & 0.040\tabularnewline
\cmidrule{2-9}
& {\scriptsize CrossEntro+TRUST} & \textbf{100.0} & {\scriptsize \bf 97.14} & {\scriptsize \bf 93.33} & {\scriptsize \bf 85.36} & {\scriptsize \bf 76.29} & \textbf{0.07} & \textbf{0.037}\tabularnewline
\midrule 
\multirow{1}{*}{{\scriptsize Imagenet}} 
& {\scriptsize CrossEntro+TRUST} & {\scriptsize \bf 92.20} & {\scriptsize \bf 88.75} & {\scriptsize \bf 87.17} & {\scriptsize \bf 86.52} & {\scriptsize \bf 85.62} & \textbf{0.11} & \textbf{0.096}\tabularnewline
\bottomrule
\end{tabular}}{\scriptsize\par}
\par\end{centering}
\caption{Accuracy over the top-$k$ test samples sorted by TRUST score, AURC, and AUSE across four datasets and model architectures (Bold: monotonic Acc, best: AURC/AUSE). }
\label{tab:top_percentage_acc}
\end{table*}

\subsubsection{Models, Baselines, and Evaluation}

We use SimpleNet (3 convolutional layers and 3 fully connected layers), SimpleNet+ (4 convolutional layers and 4 fully connected
layers), VGG11{\footnotesize ,} PreactResNet18{\footnotesize ,} ResNet18,
ResNet50 and ViT-Base (trained from scratch and pre-trained) \cite{alexey2020image}
models for conducting the experiments. CIFAR-10 dataset is trained
with all the SimpleNet, SimpleNet+, VGG11,{\footnotesize{} }PreactResNet18
and ViT-Base (trained from scratch), CAMELYON-17 was trained with
PreactResNet18, TinyImagenet is trained with ResNet50 and a pre-trained ViT-Base model was used for Imagenet. For TinyImageNet, we selected classes with $\geq$60\% accuracy (overall: 65.19\%) for further analysis. For ImageNet, we randomly chose 100 classes with moderate accuracy (80-90\%). We use the Adam optimizer \cite{Kingma_Ba_14Adam} with a learning rate of 0.001. We set the softmax temperature $T$ to $5.0$, $\lambda$ to 0.001, and the number of optimization epochs to 10k in all experiments.
The main results (Table \ref{tab:top_percentage_acc}) use PreactResNet18 for CIFAR-10 and CAMELYON-17, ResNet-50 for TinyImageNet, and pre-trained ViT-Base for ImageNet.
\begin{figure*}
    \subfloat 
{
\includegraphics[width=0.32\columnwidth]{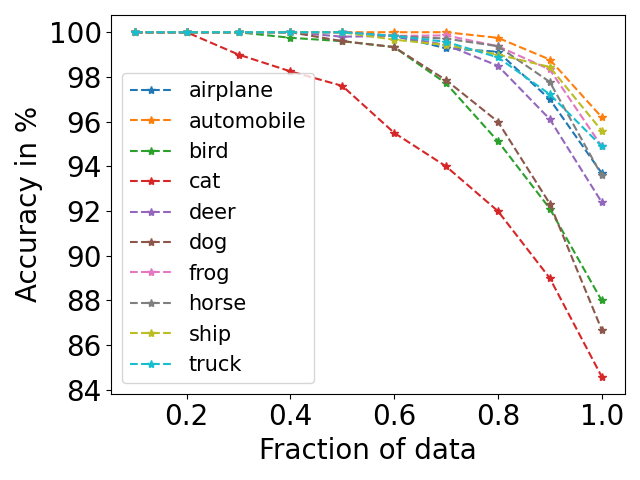}
}
    \subfloat 
{
\includegraphics[width=0.32\columnwidth]{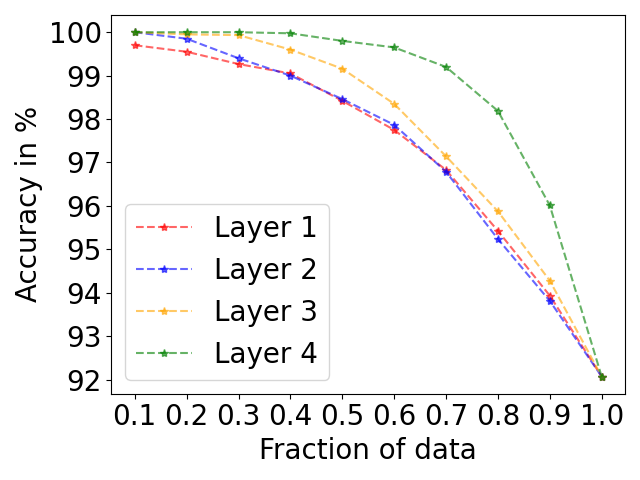}
}
\subfloat
{
\includegraphics[width=0.32\columnwidth]{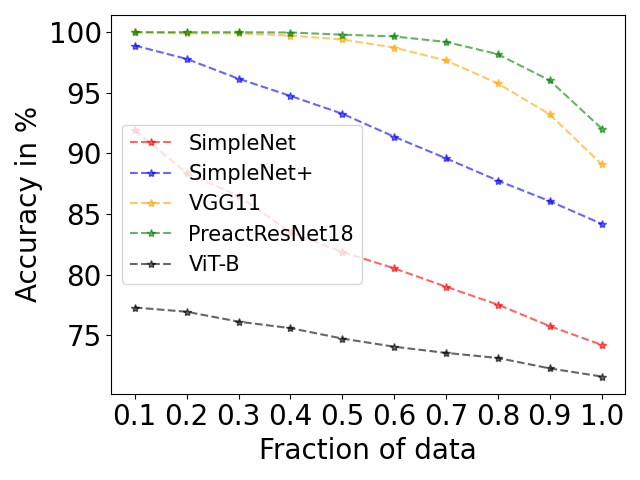}
}
\caption{Comparison of TRUST score-based accuracies on CIFAR-10: (a) class-wise, (b) layer-wise, and (c) across five (SimpleNet to ViT-B) different models.}\label{fig:cifar-10-detailed}
\end{figure*}
We added dropout \cite{gal2016dropout} (0.2 for CIFAR-10; 0.3 for CAMELYON-17 and TinyImageNet) to PreactResNet18, ResNet18, and ResNet50. For each test sample, we averaged predictions over 50 stochastic passes, using the resulting class distribution's entropy \cite{shannon1948mathematical,Gray_11Entropy} as the uncertainty measure. Other baselines
reported are Density aware \cite{yoon2024uncertainty}, ViM \cite{wang2022vim}, SIRC \cite{xia2022augmenting},
and LogitNorm \cite{wei2022mitigating}.

\textbf{Evaluation}: We report accuracy (Table \ref{tab:top_percentage_acc}) to evaluate how well TRUST ranks predictions by confidence. This accuracy is over the top-$k$ samples sorted by TRUST score. We also report standard measures such as AURC \cite{zhou2024novel} and AUSE \cite{lind2024uncertainty}
scores for a comprehensive evaluation. 

\section{Results}
\subsection{In data stratification}

Table \ref{tab:top_percentage_acc} presents the accuracy of increasingly
confident test subsets, starting from the entire data set to the top 10\%
of the most confident subset, using baselines and TRUST scores as
different measures of prediction uncertainty. We report \textbf{CrossEntro+TRUST}
for models trained with Cross Entropy loss and \textbf{LogitNorm+TRUST}
for those trained with LogitNorm loss \cite{wei2022mitigating} in
Table \ref{tab:top_percentage_acc}. Across all four datasets, the
TRUST score consistently shows the desirable property of accuracy
increasing monotonically with smaller subset sizes. In contrast, dropout-based,
ViM, SIRCs scores fail to exhibit this pattern. Even for CIFAR-10,
TRUST score reaches 99\% accuracy
at the top 70\%. Among CIFAR-10 results, Density-aware and LogitNorm
produced the most comparable results.\begin{figure*}
    \subfloat {
\includegraphics[width=0.45\columnwidth]{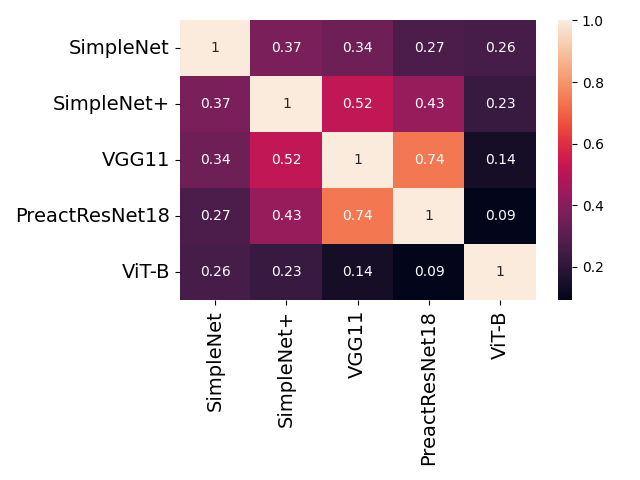}
}
\subfloat
{
\includegraphics[width=0.45\columnwidth]{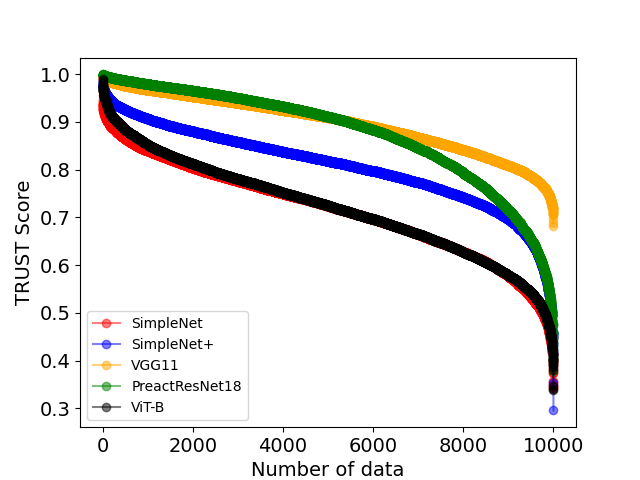}
}\caption{Comparison of model-wise TRUST score rankings and distributions on the CIFAR-10 test dataset: (a) Spearman’s rank correlation between model-specific TRUST score rankings, and (b) Sorted TRUST scores assigned by each model.}\label{fig:Spearmann's-rank-corrleation}
\end{figure*}
\begin{figure*}
\subfloat{
\includegraphics[width=0.51\columnwidth]{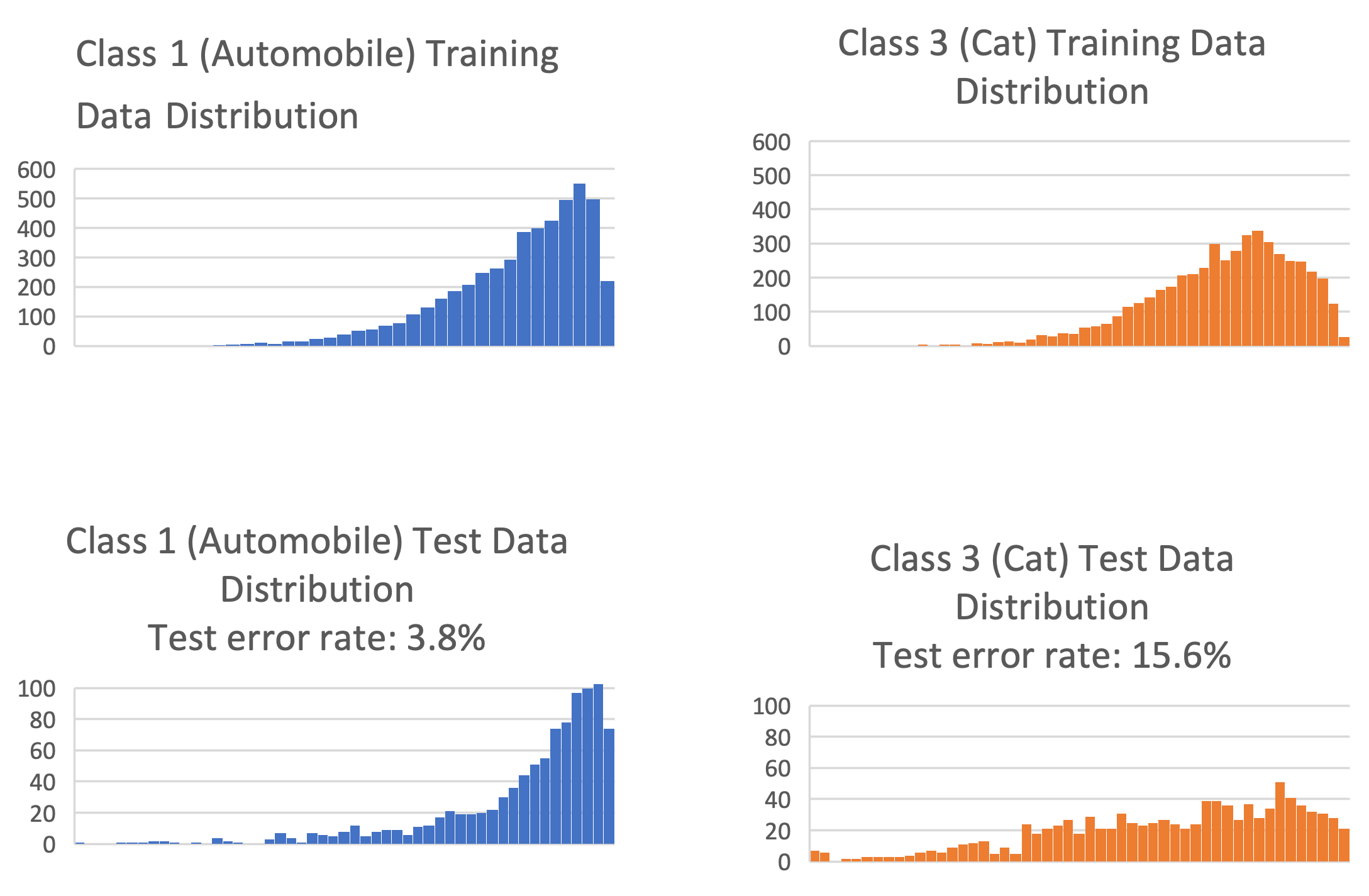}
}
\hspace{0.09\textwidth}
\subfloat{
\includegraphics[width=0.4\columnwidth]{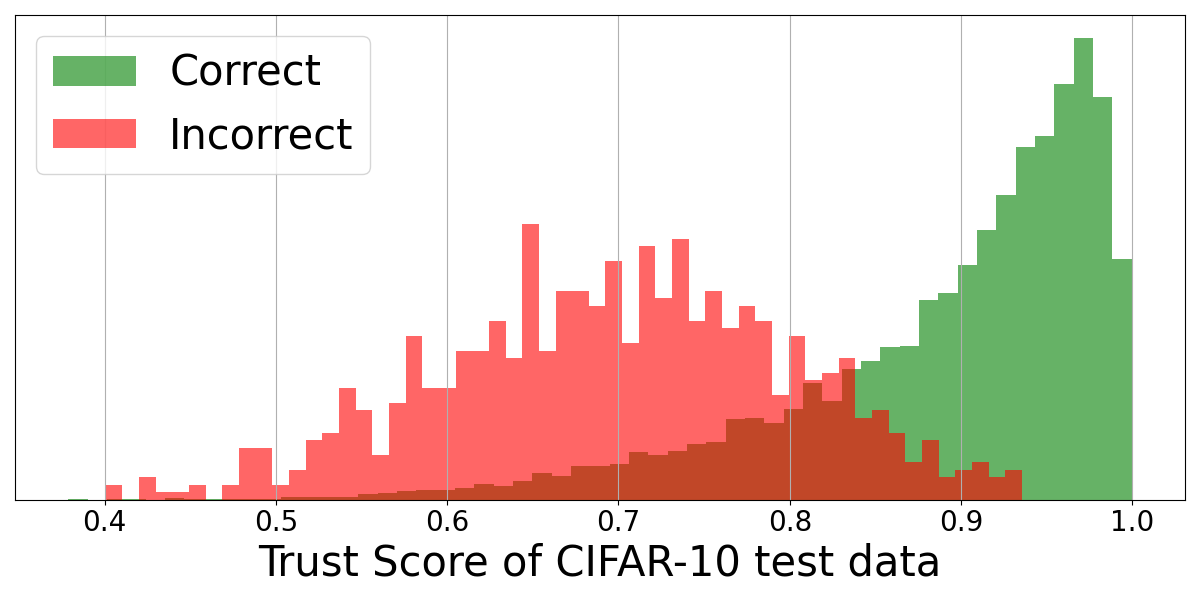}}
\caption{TRUST score distributions for CIFAR-10. (a) Training vs. test samples for automobile and cat classes of CIFAR-10. (b) Normalized histogram of CIFAR-10 test samples with correct (in green) and incorrect (in red) predictions.}.\label{fig:TRUST-score-distributions}
\end{figure*} Since LogitNorm is a training
method, we further applied TRUST to the LogitNorm-trained (LogitNorm+TRUST)
network, which led to additional improvements across all stratification
levels. This demonstrates that our method remains effective in enhancing
uncertainty quantification beyond what LogitNorm alone achieves. TinyImagenet
shows the most significant improvement, with accuracy rising from
76\% to 100\% at the top 20\% level. For a different architecture,
such as ViT-B on Imagenet and the dataset CAMELYON-17, we also observe
steady accuracy improvements with smaller subsets. These results empirically
demonstrate the utility of the TRUST metric in effectively segregating
reliable predictions from unreliable ones.
\begin{figure*}
\vspace{-16pt}
    \subfloat 
{
\includegraphics[width=0.45\columnwidth]{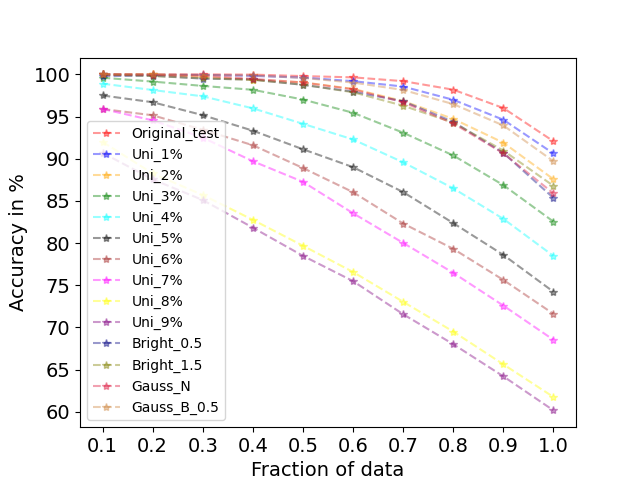}
}
\subfloat  
{
\includegraphics[width=0.45\columnwidth]{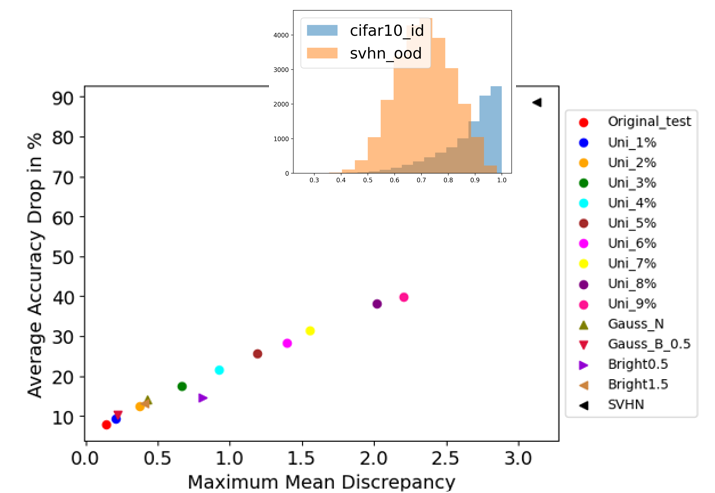}
}
\caption{(a) Accuracy vs. TRUST-stratified CIFAR-10 test samples under noise corruptions: Uni\_p\%, Bright\_b, Gauss\_N, Gauss\_B\_0.5, and Original test, (b)Accuracy drop vs. MMD for corrupted and OOD (SVHN) test sets; includes TRUST score histograms for CIFAR-10 and SVHN.}.\label{fig:noises}
\vspace{-16pt}
\end{figure*}
Fig \ref{fig:cifar-10-detailed}(a) presents stratification results
across the 10 classes for the PreactResNet18 CIFAR-10 model, highlighting
the ‘cat’ class as the most challenging. All classes, except for ‘cat’,
‘dog’, and ‘bird’, achieved nearly 100\% accuracy at the top 60\%
stratification level. This suggests that, given a representative test
set, we could set class-specific thresholds to achieve the target
accuracy at the class level, covering a broader portion of the test
data distribution than with a single aggregate threshold. For CIFAR-10,
using class-specific thresholds, we could cover approximately 52\%
of the test data with nearly 100\% accuracy significantly more than
the 30\% coverage across all classes (Table \ref{tab:top_percentage_acc},
CIFAR-10, CrossEntro+TRUST row). Interestingly, we also see from Fig \ref{fig:cifar-10-detailed}(b)
that the monotonicity is preserved even across different feature layers
of the model but the last feature layer seem to offer the best stratification. Fig \ref{fig:cifar-10-detailed}(c)
shows the accuracy versus stratification across five different models
trained on CIFAR-10. Regardless of model size (various CNNs) or type
(CNN and ViT), the TRUST score consistently provided the desired monotonic
increase in accuracy as more high-confidence predictions were selected.

\subsection{Small to large model inspection}

Here, we examine the utility of TRUST scores in identifying patterns
across different classifiers to gain insights into their behaviors.
Fig \ref{fig:Spearmann's-rank-corrleation}(a) presents the Spearman\textquoteright s
rank correlation between the sorted order of test data across different
models on the CIFAR-10 dataset, revealing several key observations:
(a) High-accuracy models, such as PreactResNet18 and VGG11, exhibit
strong rank-order agreement, indicated by their high correlations.
This suggests that high-performing models may learn similar patterns,
resulting in comparable ranking. (b) SimpleNet+ is more correlated
with VGG11 than with PreactResNet18, suggesting that ResNet-based
models might learn slightly distinct features compared to simpler
CNN models. (c) ViT stands out as the least correlated model, likely
due to its unique architecture, which leads it to learn different
patterns.

Fig \ref{fig:Spearmann's-rank-corrleation}(b) shows the sorted TRUST scores produced
by various models on the CIFAR-10 test dataset. Although PreactResNet18
and VGG11 achieve similar accuracy levels, PreactResNet18 demonstrates
better calibration, likely due to its robustness in capturing tail
data points distinct from the concentrated modes. In contrast, simpler
models display a more gradual decline in concentration around the
mode and a broader spread, indicating a smoother function compared
to VGG11 and PreactResNet18, even though SimpleNet+ achieves accuracy
similar to VGG11. In summary, the distribution of TRUST scores provides
valuable insights into how data is represented in a model\textquoteright s
feature space, guiding us toward selecting models with better-calibrated
and reliable confidence scores. 
\subsection{Understanding Data Alignment through TRUST}
We use TRUST scores to assess alignment between training and test data and its impact on accuracy. Fig \ref{fig:TRUST-score-distributions}(a) shows that classes with high train-test TRUST score overlap (e.g., ‘airplane’) have lower test error, while those with low overlap (e.g., ‘cat’) show higher error. We also observe class-specific variance and tail-driven errors. Fig \ref{fig:TRUST-score-distributions}(b) highlights how TRUST ranks correct predictions higher than the incorrect ones. We further test on noise-corrupted CIFAR-10 and OOD data (SVHN), finding that accuracy drop correlates linearly with TRUST distribution divergence (MMD) from the training set (Fig \ref{fig:noises}(a)), and the monotonic accuracy trend persists across corruptions (Fig \ref{fig:noises}(b)).
\begin{figure*}
\subfloat [\label{fig:a-1}] 
{
\includegraphics[width=0.32\columnwidth]{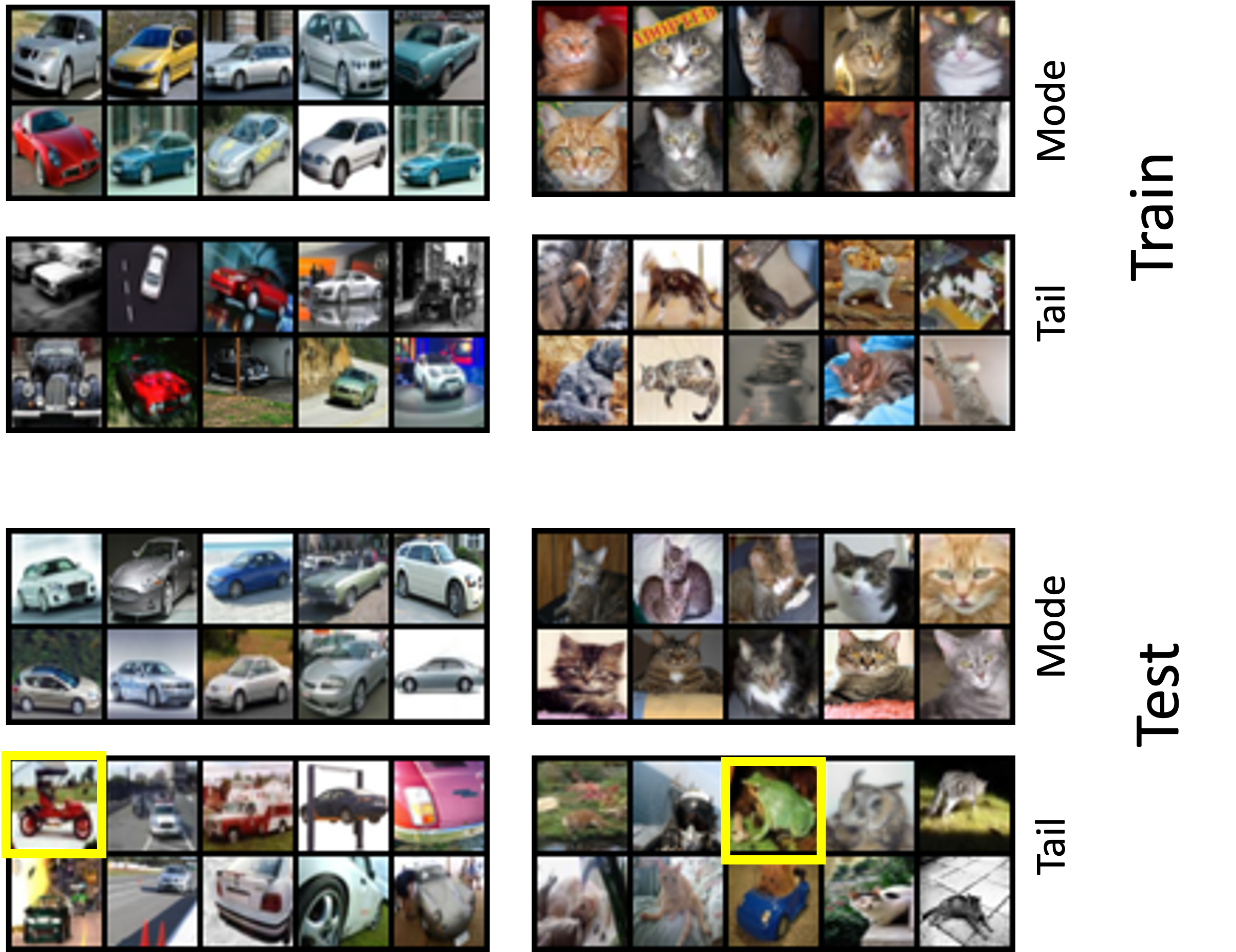}
}
\subfloat [\label{fig:b-1}]
{
\includegraphics[width=0.32\columnwidth]{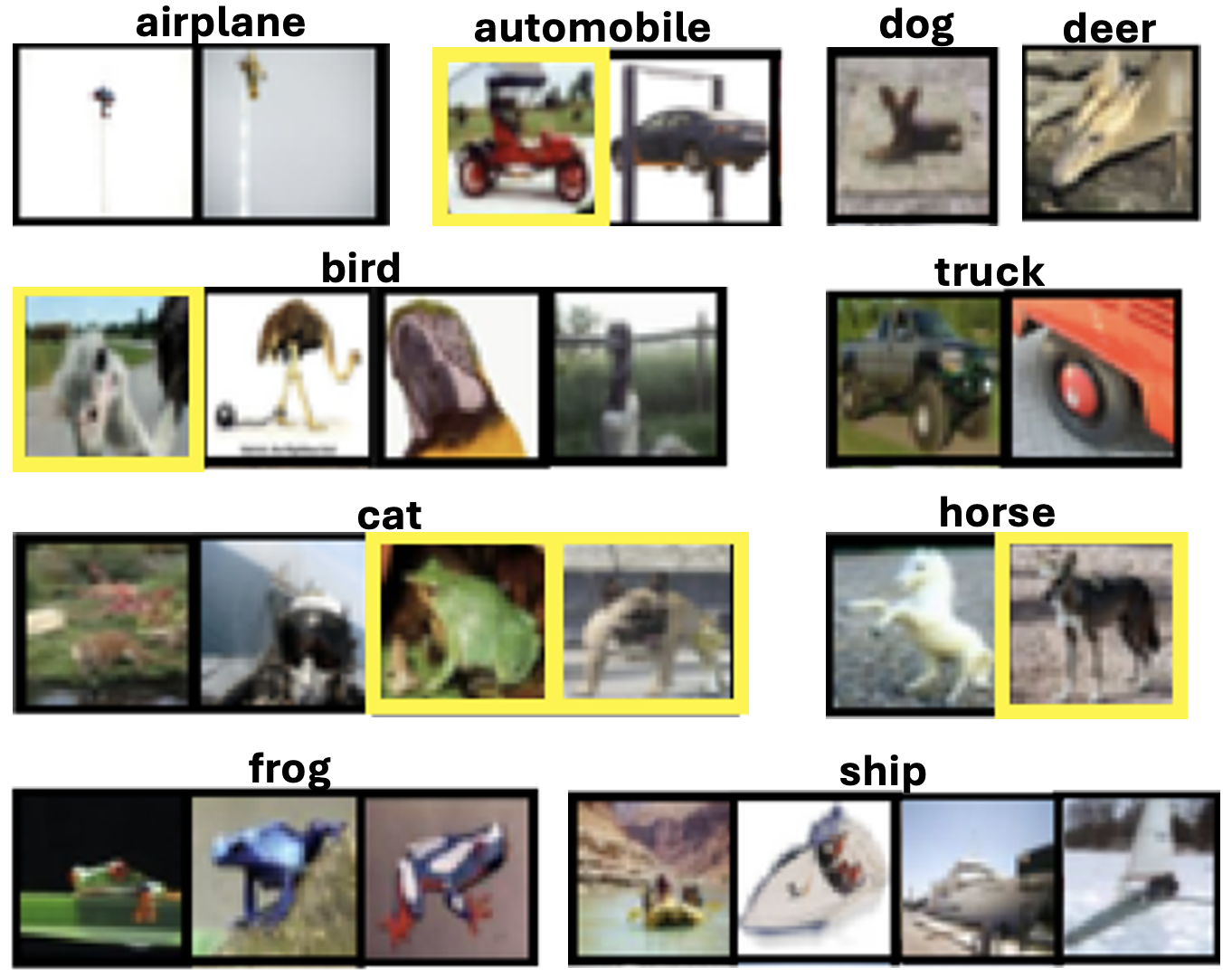}
}
\subfloat [\label{fig:b-1}]
{
\includegraphics[width=0.34\columnwidth]{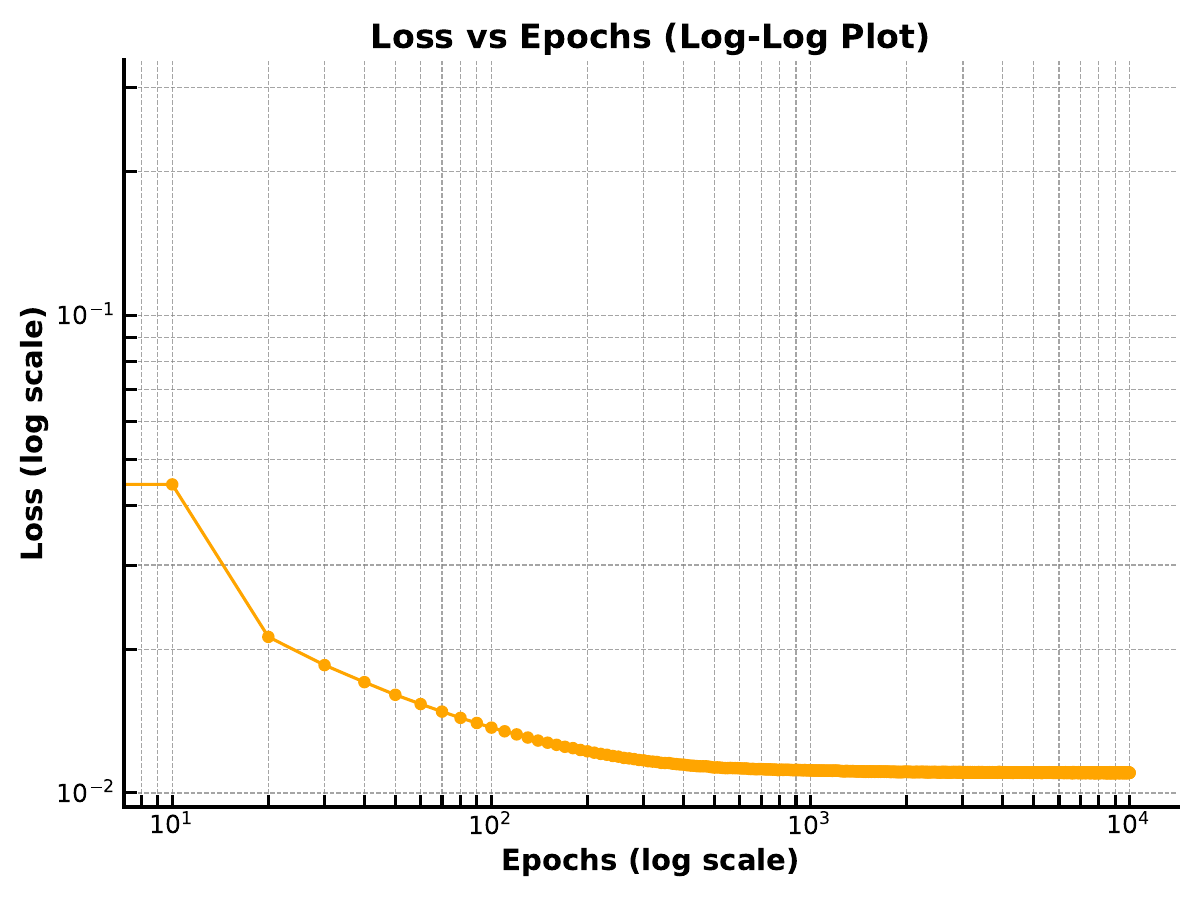}
}
\caption{TRUST score-based mode and tail samples from CIFAR-10: (a) ‘Automobile’ and ‘Cat’ class examples from train and test sets; (b) mislabelled and rare test samples across all classes; (c) Convergence plot for our chosen $T$ (5.0) and $\lambda$ (0.001) values (ablation are in supplementary).}\label{fig:wrong}
\vspace{-3pt}
\end{figure*}

We analyze CIFAR-10 training and test images based on their TRUST scores to identify typical (high-score) and rare (low-score) samples. Fig \ref{fig:wrong}(a) illustrates this for the ‘automobile’ and ‘cat’ classes, where high-score samples represent common patterns, while low-score ones capture rare poses and mislabeled samples in the original dataset (highlighted). Fig \ref{fig:wrong}(b) shows similar rare / mislabeled test samples in all classes. Fig \ref{fig:wrong}(c) shows that the convergence happens much before our 10k iteration steps. 

\noindent\textbf{Computational cost and limitation:} We use the V100 GPUs to run all the experiments and it may take several seconds per sample to compute the TRUST score making it unsuitable for real-time application. Further, TRUST’s effectiveness diminishes with smaller architectures and the current work only consider image data.

\noindent\textbf{Broader Impact:} Our work contributes positively by making machine learning more trustworthy.

%% file: sec/6_conclusion.tex
\section{Conclusion}

\label{sec:conclusion}

In this paper, we first analyze why current methods often produce
noisy measures of epistemic uncertainty, and then propose a new test-time
optimization-based method for achieving significantly improved estimates.
We introduce a new measure, the TRUST score, which aligns closely with
epistemic uncertainty, and demonstrate its utility across a diverse
range of tasks, including data stratification, assessing alignment
between training and test data, dataset inspection, and deriving insights
into model behavior when tested across four benchmark datasets using a multitude of different model architectures.

%% file: sec/suppl.tex
\section*{Theorem details}

\subsection{Proofs}
\begin{theorem}
(\textbf{Concentration of Measure on the Hypersphere}):
\end{theorem}

Let $x$ be a random vector in $\mathbb{R}^{d}$, where each component
$x_{i}$ is independently drawn from a distribution with mean $0$
and variance $\sigma^{2}$. Define the Euclidean norm of $x$ as $||x||=\sqrt{\sum_{i=1}^{d}x_{i}^{2}}$.
\begin{enumerate}
\item \textbf{Norm Concentration: }As the dimensionality $d\to\infty$,
the Euclidean norm $||x||$ concentrates around $\sqrt{d}\sigma$
, meaning that for any small $\epsilon>0$ , $Pr\left(\left|||x||-\sqrt{d}\sigma\right|<\epsilon\sqrt{d}\sigma\right)\to1$.
\item \textbf{Surface Concentration:} Consequently, as $d$ grows large,
the points $x$ become increasingly concentrated near the surface
of a hypersphere with radius $\sqrt{d}\sigma$ centered at the origin.
Specifically, the probability that a randomly chosen point lies within
a thin shell of radius $\sqrt{d}\sigma\pm\epsilon$ approaches 1 as
$d\to\infty$ .
\end{enumerate}
\begin{proof}
\textbf{1. Norm Concentration}

Let $x$ be a random vector in $\mathbb{R}^{d}$, where each component
$x_{i}$ is independently drawn from a distribution with mean $0$
and variance $\sigma^{2}$. The Euclidean norm of $x$ is given by:
$\|x\|=\sqrt{\sum_{i=1}^{d}x_{i}^{2}}$.

Define $S=\sum_{i=1}^{d}x_{i}^{2}$. The expectation of $S$ is:

\[
\mathbb{E}[S]=\sum_{i=1}^{d}\mathbb{E}[x_{i}^{2}]=d\sigma^{2}
\]
.

The variance of $S$ can be computed as:

\[
\text{Var}(S)=\sum_{i=1}^{d}\text{Var}(x_{i}^{2})
\]
,

where for each $x_{i}^{2}$, using the properties of variance: $\text{Var}(x_{i}^{2})=\mathbb{E}[x_{i}^{4}]-(\mathbb{E}[x_{i}^{2}])^{2}$
and assuming $x_{i}$ comes from a distribution with finite fourth
moment, we denote $\mathbb{E}[x_{i}^{4}]$ as $\mu_{4}$, we can rewrite:

\[
\text{Var}(S)=d(\mu_{4}-\sigma^{4})
\]
.

By Chebyshev\textquoteright s inequality, the probability that $S$
deviates from its expectation is bounded as:

\[
Pr\left(\left|S-d\sigma^{2}\right|\geq\epsilon d\sigma^{2}\right)\leq\frac{\text{Var}(S)}{(\epsilon d\sigma^{2})^{2}}=\frac{d(\mu_{4}-\sigma^{4})}{\epsilon^{2}d^{2}\sigma^{4}}
\]
.

As $d\to\infty$ , the right-hand side tends to $0$, implying: $S\to d\sigma^{2}\quad\text{with high probability}$.

Taking the square root, the norm $\|x\|=\sqrt{S}$ concentrates around
$\sqrt{d}\sigma$. More formally, for any $\epsilon>0$ : $\Pr\left(\left|\|x\|-\sqrt{d}\sigma\right|<\epsilon\sqrt{d}\sigma\right)\to1\quad\text{as }d\to\infty$.

\textbf{2. Surface Concentration}

From the norm concentration result, we know that the Euclidean norm
$\|x\|$ is highly likely to lie in the interval $[\sqrt{d}\sigma-\epsilon\sqrt{d}\sigma,\sqrt{d}\sigma+\epsilon\sqrt{d}\sigma]$
. This implies that the random vector $x$ is concentrated within
a thin shell of radius $\sqrt{d}\sigma\pm\epsilon\sqrt{d}\sigma$
. 

More formally, let $r=\|x\|$ and consider the probability that $x$
lies within a shell of width $2\epsilon\sqrt{d}\sigma$ : $Pr\left(\sqrt{d}\sigma-\epsilon\sqrt{d}\sigma\leq r\leq\sqrt{d}\sigma+\epsilon\sqrt{d}\sigma\right)$.
Using the norm concentration derived above, this probability approaches
1 as $d\to\infty$. The geometric interpretation is that, as $d$
grows large, most of the probability mass for the random vector $x$
is concentrated on the surface of a hypersphere with radius $\sqrt{d}\sigma$
.
\end{proof}
\begin{theorem}

Let $n$ points be independently and uniformly distributed on the
surface of a $d$-dimensional unit hypersphere. Then, the expected
minimum value of $\cos(\omega)$ , where $\omega$ is the angle between
any pair of points, is approximately given by:
$E[\cos(\omega_{\text{min}})]\approx-\sqrt{\frac{2\ln n}{d}}$

where $\cos(\omega_{\text{min}})$ denotes the minimum cosine value
over all pairs of points under the assumption of both $n$ and $d$
being large.
\end{theorem}

\begin{proof}
Let $n$ points be independently and uniformly distributed on the
surface of a $d$-dimensional unit hypersphere $\mathcal{S}^{d-1}$.
For two points $x_{1},x_{2}$ on $\mathcal{S}^{d-1}$, the cosine
of the angle $\omega$ between them is given by:

\[
\cos(\omega)=x_{1}\cdot x_{2}=\sum_{i=1}^{d}x_{1i}x_{2i}
\]

where $x_{1}\cdot x_{2}$ is the dot product of the two points. We
seek the expected minimum value of $\cos(\omega)$ over all pairs
of points: $E[\cos(\omega_{\text{min}})]\approx-\sqrt{\frac{2\ln n}{d}}$
under the assumption that n and d are large.

\textbf{Step 1: Distribution of $\cos(\omega)$}

On a $d$ -dimensional hypersphere, if two points are independently
and uniformly distributed, the dot product $x_{1}\cdot x_{2}$ (or
equivalently $\cos(\omega)$ ) is approximately Gaussian for large
$d$ due to the Central Limit Theorem. Specifically: $\cos(\omega)\sim\mathcal{N}\left(0,\frac{1}{d}\right)$,
where the mean is $0$ (due to symmetry) and the variance is $\frac{1}{d}$
because each component $x_{1i}x_{2i}$ contributes $\frac{1}{d}$
to the variance.

\textbf{Step 2: Minimum of Pairwise Cosines}

The number of unique pairs among $n$ points is:

\[
\binom{n}{2}\approx\frac{n^{2}}{2}
\]

For large $n$, the minimum value of $\cos(\omega)$ is determined
by the smallest value among these $\frac{n^{2}}{2}$ pairwise cosines.
The probability that any single cosine value is less than a threshold
$t$ is given by the cumulative distribution function (CDF) of a standard
normal distribution scaled by $\sqrt{d}$, which is:

\[
P(\cos(\omega)<t)\approx\Phi(t\sqrt{d})
\]

where $\Phi(z)$ is the CDF of a standard normal distribution. For
the smallest cosine $\cos(\omega_{\text{min}})$, the complementary
probability that no cosine value is less than $t$ is:

\[
P(\cos(\omega_{\text{min}})\geq t)\approx\left[1-\Phi(t\sqrt{d})\right]^{\frac{n^{2}}{2}}
\]

Taking the logarithm to simplify:

\[
\ln\left(P(\cos(\omega_{\text{min}})\geq t)\right)\approx\frac{n^{2}}{2}\ln\left(1-\Phi(t\sqrt{d})\right)
\]

For large $n$ , $\Phi(t\sqrt{d})$ becomes small, so we approximate
$\ln(1-\Phi(t\sqrt{d}))\approx-\Phi(t\sqrt{d})$ . Substituting:

\[
\ln\left(P(\cos(\omega_{\text{min}})\geq t)\right)\approx-\frac{n^{2}}{2}\Phi(t\sqrt{d})
\]

\textbf{Step 3: Approximation for the Minimum}

The expected minimum cosine value corresponds to the $t$ where:

\[
P(\cos(\omega_{\text{min}})\geq t)\approx e^{-1}
\]

Setting the above probability to $e^{-1}$ gives:

\[
\frac{n^{2}}{2}\Phi(t\sqrt{d})\approx1
\]

Solving for $t$ :

\[
\Phi(t\sqrt{d})\approx\frac{2}{n^{2}}
\]

For small arguments, the inverse CDF $\Phi^{-1}(p)$ of a Gaussian
distribution satisfies $\Phi^{-1}(p)\approx\sqrt{2\ln\frac{1}{p}}$
. Substituting $\Phi(t\sqrt{d})\approx\frac{2}{n^{2}}$:

\[
t\sqrt{d}\approx-\sqrt{2\ln n}
\]
.

Dividing through by $\sqrt{d}$, we find:

\[
t\approx-\sqrt{\frac{2\ln n}{d}}
\]

Thus, the expected minimum cosine value is approximately:

\[
E[\cos(\omega_{\text{min}})]\approx-\sqrt{\frac{2\ln n}{d}}
\]
\end{proof}
\begin{lemma}
Let $x_{1},x_{2},\ldots,x_{n}$ be a set of $n$ items, each associated
with a true score $s_{i}$ for $i=1,2,\dots,n$ . Suppose the observed
score $\tilde{s_{i}}$ of each item $x_{i}$ is corrupted by Gaussian
noise $\varepsilon_{i}$ with mean zero and standard deviation $\sigma$
, such that:

\[
\tilde{s_{i}}=s_{i}+\varepsilon_{i},\quad\text{where }\varepsilon_{i}\sim\mathcal{N}(0,\sigma^{2})
\]

Let $\Delta s_{ij}=s_{i}-s_{j}$ represent the true difference in
scores between items $i$ and $j$, and let $\sigma_{ij}^{2}=2\sigma^{2}$
denote the variance of the difference $\Delta\tilde{s_{ij}}=\tilde{s_{i}}-\tilde{s_{j}}$
due to noise. Define a sorting error to occur if the observed ordering
based on $s_{i}{\prime}$ differs from the true ordering based on
$s_{i}$. Then, the probability $P_{ij}$ of a sorting error between
items $i$ and $j$ (i.e., $\tilde{s_{i}}<\tilde{s_{j}}$ when $s_{i}>s_{j}$)
is given by:

\[
P_{ij}=P(\tilde{s_{i}}<\tilde{s_{j}})=1-\Phi\left(\frac{\Delta s_{ij}}{\sqrt{2}\sigma}\right)
\]

where $\Phi$ is the cumulative distribution function of the standard
normal distribution.
\end{lemma}

\begin{proof}
Let $x_{1},x_{2},\ldots,x_{n}$ represent a set of $n$ items, where
each item $x_{i}$ has a true score $s_{i}$. The observed score $\tilde{s}_{i}$
of each item is given by: $\tilde{s}_{i}=s_{i}+\varepsilon_{i}$,
where $\varepsilon_{i}\sim\mathcal{N}(0,\sigma^{2})$ represents Gaussian
noise with mean 0 and variance $\sigma^{2}$. Let:

\textbullet{} $\Delta s_{ij}=s_{i}-s_{j}$ represent the true difference
in scores.

\textbullet{} $\Delta\tilde{s}_{ij}=\tilde{s}_{i}-\tilde{s}_{j}$
represent the observed difference in scores.

\textbullet{} A sorting error occurs if $\tilde{s}_{i}<\tilde{s}_{j}$
when $s_{i}>s_{j}$.

We aim to compute the probability $P_{ij}$ of a sorting error, i.e.,
$P(\tilde{s}_{i}<\tilde{s}_{j})$ , and show that it equals:

\[
P{ij}=1-\Phi\left(\frac{\Delta s{ij}}{\sqrt{2}\sigma}\right)
\]
,

where $\Phi(.)$ is the CDF of the standard normal distribution.

\textbf{Step 1: Distribution of }$\Delta\tilde{s}_{ij}$

The observed score difference is given by:

\[
\Delta\tilde{s}_{ij}=\tilde{s}_{i}-\tilde{s}j=(s_{i}+\varepsilon_{i})-(s_{j}+\varepsilon_{j})=\Delta s{ij}+(\varepsilon_{i}-\varepsilon_{j})
\]

Since $\varepsilon_{i}$ and $\varepsilon_{j}$ are independent Gaussian
random variables with mean 0 and variance $\sigma^{2}$, their difference
$\varepsilon_{i}-\varepsilon_{j}$ is also Gaussian with: $\mathbb{E}[\varepsilon_{i}-\varepsilon_{j}]=0$,

\[
\text{Var}(\varepsilon_{i}-\varepsilon_{j})=\text{Var}(\varepsilon_{i})+\text{Var}(\varepsilon_{j})=\sigma^{2}+\sigma^{2}=2\sigma^{2}
\]

Thus, $\Delta\tilde{s}{ij}$ is distributed as $\sim\mathcal{N}(\Delta s_{ij},2\sigma^{2})$.

\textbf{Step 2: Probability of a Sorting Error}

A sorting error occurs if $\tilde{s}_{i}<\tilde{s}_{j}$ when $s_{i}>s_{j}$.
Equivalently, this is the event $\Delta\tilde{s}{ij}<0$. Using the
distribution of $\Delta\tilde{s}{ij}$, the probability of this event
is $P_{ij}=P(\Delta\tilde{s}_{ij}<0)$.

Standardizing the random variable $\Delta\tilde{s}{ij}$ , we define
a standard normal variable $z=\frac{\Delta\tilde{s}{ij}-\Delta s_{ij}}{\sqrt{2}\sigma}$,
which follows $z\sim\mathcal{N}(0,1)$ . The probability of a sorting
error becomes:

\[
P_{ij}=P(\Delta\tilde{s}{ij}<0)=P\left(Z<\frac{0-\Delta s{ij}}{\sqrt{2}\sigma}\right)=P\left(Z<-\frac{\Delta s_{ij}}{\sqrt{2}\sigma}\right)
\]

Using the symmetry of the standard normal distribution, $P(z<-z)=1-\Phi(z)$
, we have:

\[
P_{ij}=1-\Phi\left(\frac{\Delta s_{ij}}{\sqrt{2}\sigma}\right)
\]
\end{proof}
\begin{theorem}
Let s$=\cos(\omega)$ be the cosine similarity between a test data
point $x_{\text{test}}$ and its nearest mode $x_{\text{test}}^{\text{mode}}$,
with angle $\omega$ between them. Suppose Gaussian noise $\delta\theta\sim\mathcal{N}(0,\sigma_{\omega}^{2})$
is added to $\omega$, resulting in a noisy score $\tilde{s}=\cos(\omega+\delta\omega)$.
For an equivalent score function $s$ with direct Gaussian noise $\delta\omega\sim\mathcal{N}(0,\sigma_{\omega}^{2})$
, the probability $P_{\text{cos}}$ of a sorting error in the cosine
distance scenario is upper-bounded by the probability $P_{\text{direct}}$
of a sorting error in the direct noise scenario:

\[
P_{\text{cos}}\leq P_{\text{direct}}
\]

where the bound holds because $\sigma_{s}^{2}=\sin^{2}(\omega)\cdot\sigma_{\omega}^{2}\leq\sigma_{\omega}^{2}$,
with equality when $\omega=\frac{\pi}{2}$ .
\end{theorem}

\begin{proof}
Let $s=\cos(\omega)$ represent the cosine similarity between a test
data point $x_{\text{test}}$ and its nearest mode $x_{\text{test}}^{\text{mode}}$
, with $\omega$ as the angle between them. Adding Gaussian noise
$\delta\omega\sim\mathcal{N}(0,\sigma_{\omega}^{2})$ to $\omega$
results in a noisy score:

\[
\tilde{s}=\cos(\omega+\delta\omega)
\]

For an equivalent scenario with direct noise added to $s$ , the noisy
score is:

\[
\tilde{s}=s+\delta s,\quad\text{where }\delta s\sim\mathcal{N}(0,\sigma_{s}^{2})
\]

We aim to show that the probability of a sorting error in the cosine
distance scenario, $P_{\text{cos}}$, is upper-bounded by the probability
of a sorting error in the direct noise scenario, $P_{\text{direct}}$,
due to $\sigma_{s}^{2}=\sin^{2}(\omega)\cdot\sigma_{\omega}^{2}\leq\sigma_{\omega}^{2}$,
with equality when $\omega=\frac{\pi}{2}$ .

\textbf{Step 1: Taylor Expansion for} $\cos(\omega+\delta\omega)$

Using the Taylor expansion of $\cos(\omega+\delta\omega)$ around
$\omega$, we write:

\[
\cos(\omega+\delta\omega)\approx\cos(\omega)-\sin(\omega)\cdot\delta\omega-\frac{1}{2}\cos(\omega)\cdot(\delta\omega)^{2}+\cdots
\]

The dominant term affected by the noise $\delta\omega$ is:

$\tilde{s}\approx s-\sin(\omega)\cdot\delta\omega$,

where $s=\cos(\omega)$ . Thus, the effective noise in $\tilde{s}$
is approximately:

\[
\delta s\approx-\sin(\omega)\cdot\delta\omega
\]

Since $\delta\omega\sim\mathcal{N}(0,\sigma_{\omega}^{2})$, it follows
that $\delta s$ is Gaussian with mean 0 and variance $\sigma_{s}^{2}=\sin^{2}(\omega)\cdot\sigma_{\omega}^{2}$.

\textbf{Step 2: Probability of Sorting Error}

A sorting error occurs when the noisy score $\tilde{s}$ reverses
the true ordering of scores. Let two items have true scores $s_{1}$
and $s_{2}$ such that $s_{1}>s_{2}$ . Sorting errors occur when
$s_{1}{\prime}<s_{2}{\prime}$ , or equivalently:

\[
\delta s_{1}-\delta s_{2}>s_{1}-s_{2}
\]

For the cosine distance scenario, the variance of $\delta s$ is reduced
by the factor $\sin^{2}(\omega)$ compared to direct noise. Specifically:

\[
\sigma_{s}^{2}=\sin^{2}(\omega)\cdot\sigma_{\omega}^{2}\leq\sigma_{\omega}^{2}
\]

As the variance of noise decreases, the probability of large deviations
from the true score ordering also decreases. Therefore, the probability
of a sorting error in the cosine distance scenario, $P_{\text{cos}}$,
is upper-bounded by the probability of a sorting error in the direct
noise scenario, $P_{\text{direct}}$, where $P_{\text{cos}}\leq P_{\text{direct}}$.
Equality holds when $\sin^{2}(\omega)=1$ , which occurs when $\omega=\frac{\pi}{2}$. 
\end{proof}
\section{Ablation Studies}
\subsection{Effect of Softmax Temperature ($T$) and Regularization ($\lambda$) on Loss Convergence}
We conducted ablation studies by varying  the softmax temperature $T$ from $5.0$ to $100.0$ as shown in Fig \ref{fig:convergence}(a) and $\lambda$ from $0.0001$ to $0.1$ in Fig \ref{fig:convergence}(b), with the number of optimization epochs set to 10k. The set $\lambda$ for varying $T$ is $0.001$ and the set softmax temperature $T$ for varying $\lambda$ is $5$. While the final convergence loss values differ across settings, both figures (Fig \ref{fig:convergence}(a) and Fig \ref{fig:convergence}(b)) clearly show that convergence typically occurs within the first few hundred epochs well before the 10k mark. This indicates that our proposed method, TRUST, converges efficiently and is computationally less intensive.

\begin{figure*}
\subfloat [\label{fig:Varying t}]
{
\includegraphics[width=0.5\columnwidth]{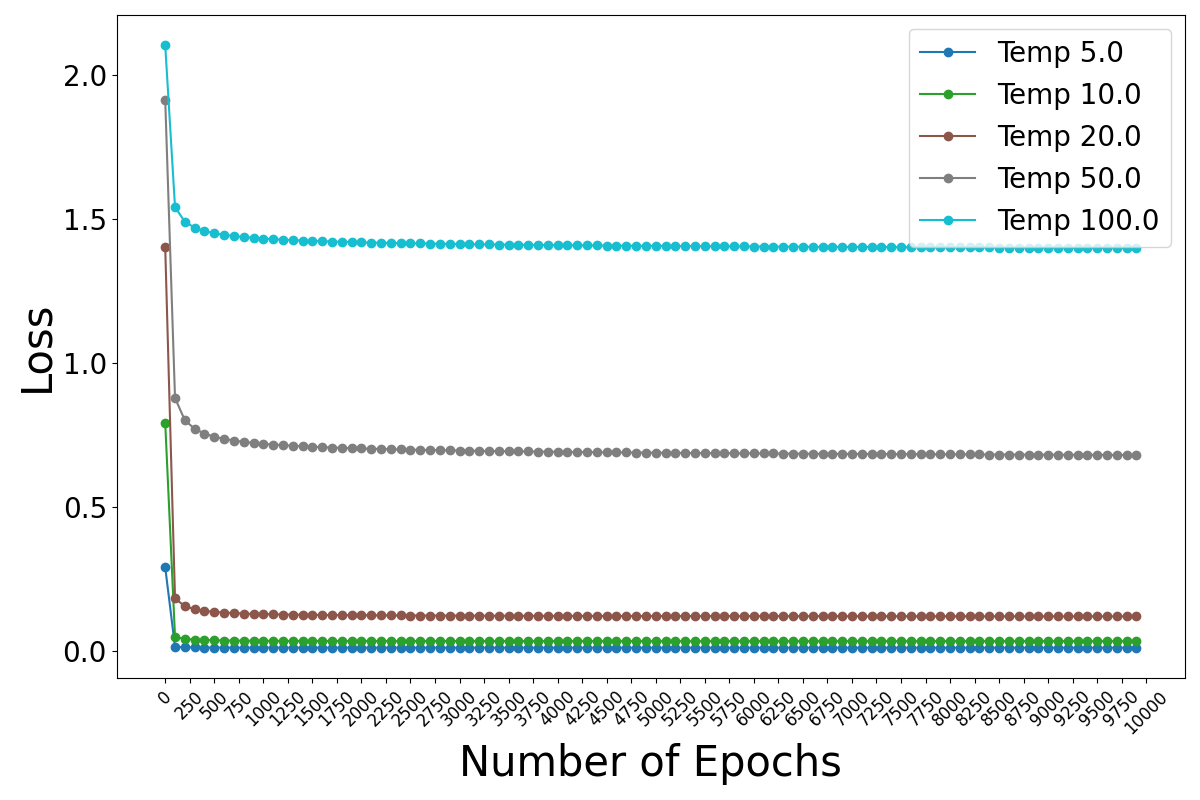}
}
\subfloat [\label{fig:Varying lambda}]
{
\includegraphics[width=0.5\columnwidth]{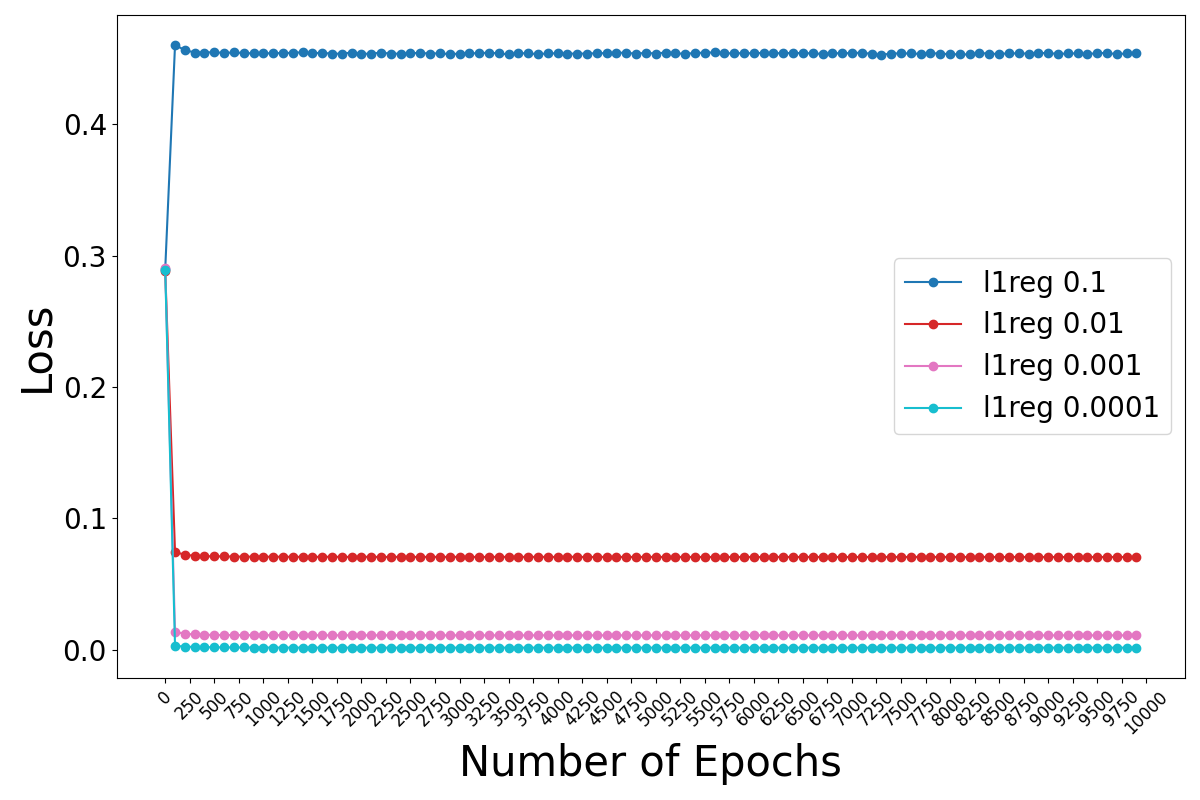}
}
\caption{Convergence of the loss on the CIFAR-10 dataset for varying softmax temperature ($T$) and regularization parameter ($\lambda$).}\label{fig:convergence}
\end{figure*}

\begin{figure*}
\subfloat [\label{fig:b-1}]
{
\includegraphics[width=0.5\columnwidth]{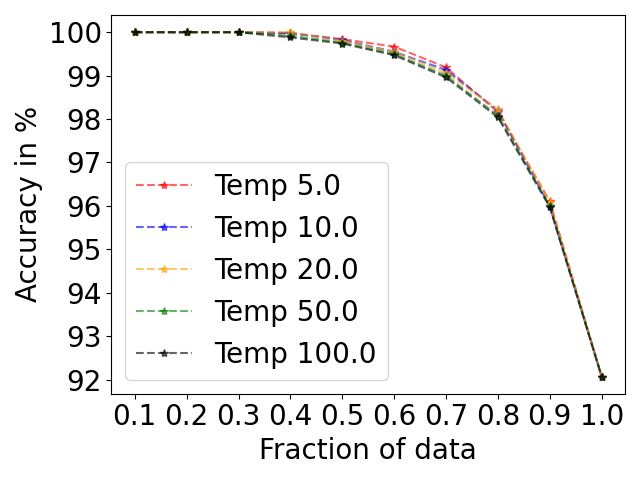}
}
\subfloat [\label{fig:b-1}]
{
\includegraphics[width=0.5\columnwidth]{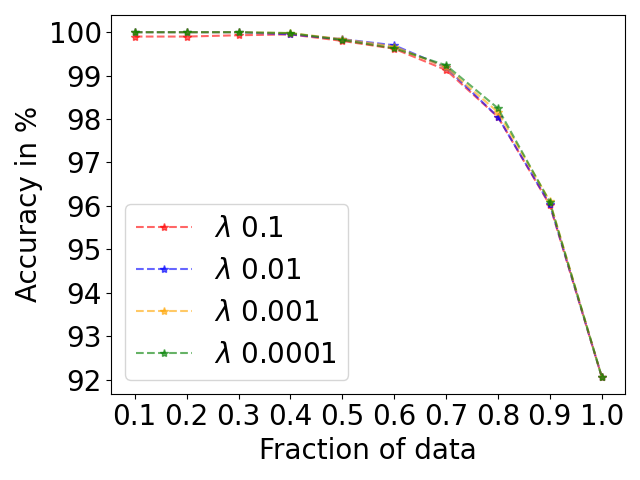}
}
\caption{Effect of varying softmax temperature ($T$) and regularization parameter ($\lambda$) on TRUST score-based accuracies for CIFAR-10.}\label{fig:accs_varyall}
\end{figure*}
\subsection{Effect of Softmax Temperature ($T$) and Regularization ($\lambda$) on TRUST Score based Accuracy}
We computed the accuracy over increasingly confident subsets of the test set starting from the full dataset and progressively narrowing down to the top 10\% most confident samples based on TRUST scores. The results, shown in Fig \ref{fig:accs_varyall}(a) and Fig \ref{fig:accs_varyall}(b), correspond to experiments where the softmax temperature $T$ was varied from 5.0 to 100.0 with $\lambda$ fixed at 0.001, and where $\lambda$ was varied from 0.0001 to 0.1 with $T$ fixed at 5.0, respectively.

From both figures, it is evident that variations in $T$ and $\lambda$ have only a minor impact on accuracy. A slight decrease in accuracy is observed at higher temperatures (e.g., $T = 100.0$ in Fig \ref{fig:accs_varyall}(a)), and similarly, higher values of $\lambda$ lead to modest accuracy drops in Fig \ref{fig:accs_varyall}(b). These declines are expected, as the corresponding configurations exhibit higher convergence loss values, as previously shown in Fig \ref{fig:convergence}(a) and Fig \ref{fig:convergence}(b).

\subsection{Effect of Fixed $T = 5.0$ and $\lambda = 0.001$ on TRUST Score based Accuracy}

We computed the accuracy over increasingly confident subsets of the test set ranging from the full dataset to the top 10\% most confident samples using TRUST scores. The plots in Fig.\ref{fig:accs_fixtlambda}(a) and Fig.\ref{fig:accs_fixtlambda}(b) are based on our selected hyperparameters: softmax temperature $T = 5.0$ and regularization parameter $\lambda = 0.001$. These plots show accuracy trends across the first 100 training epochs (in increments of 10) and the first 1000 epochs (in increments of 100), respectively. The corresponding numerical values for Fig.\ref{fig:accs_fixtlambda}(a) are presented in Table\ref{tab:top_percentage_acc_epoch100}, while the overall performance across all trained models and datasets is reported in Table~\ref{tab:Performance}. These results demonstrate early convergence, which in turn reduces additional computational overhead.

\begin{figure*}
\subfloat [\label{fig:trust_epoch100}]
{
\includegraphics[width=0.5\columnwidth]{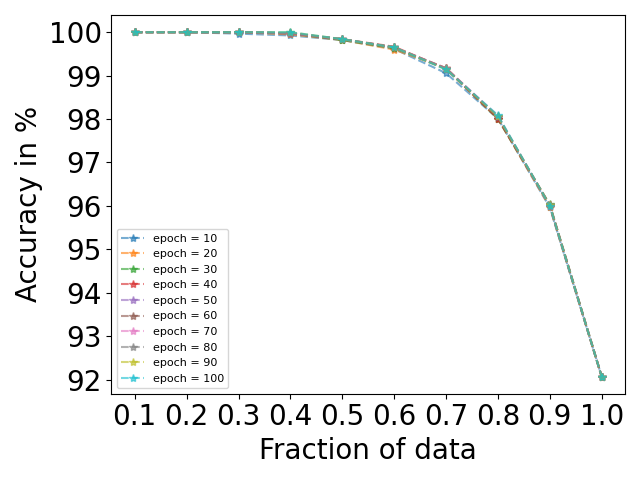}
}
\subfloat [\label{fig:trust_epoch1000}]
{
\includegraphics[width=0.5\columnwidth]{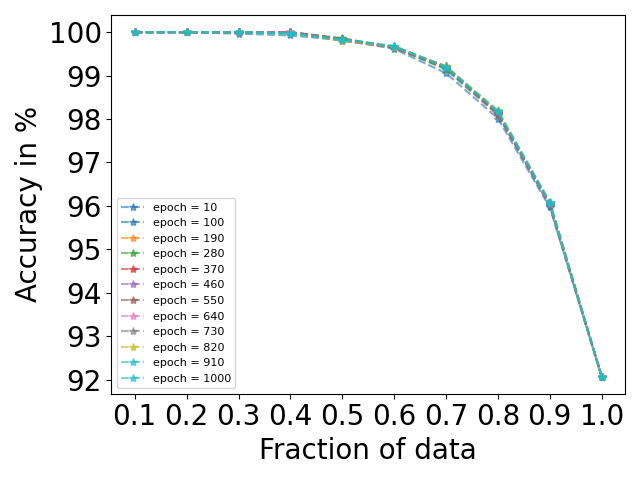}
}
\caption{TRUST score-based accuracy comparison on CIFAR-10 with $T = 5.0$ and $\lambda = 0.001$: (a) from 10 to 100 training epochs (in increments of 10), and (b) from 100 to 1000 training epochs (in increments of 100).}\label{fig:accs_fixtlambda}
\end{figure*}

\begin{table*}
\begin{centering}
{\scriptsize{}%
\begin{tabular}{p{1.75cm}cccccccccc}
\toprule 
\multirow{2}{*}{{\scriptsize Epoch Number}} & \multicolumn{10}{c}{{\scriptsize Accuracy @ Top-\% Data (by TRUST Score) $\uparrow$}} \tabularnewline
\cmidrule{2-11}
 & {\scriptsize 10} & {\scriptsize 20} & {\scriptsize 30} & {\scriptsize 40} & {\scriptsize 50} & {\scriptsize 60} & {\scriptsize 70} & {\scriptsize 80} & {\scriptsize 90} & {\scriptsize 100} \tabularnewline
\midrule
\cmidrule{1-11}
{\scriptsize 10} & {\scriptsize 100.0} & {\scriptsize 100.0} & {\scriptsize 99.97} & {\scriptsize 99.93} & {\scriptsize 99.82} & {\scriptsize 99.62} & {\scriptsize 99.06}  & {\scriptsize 98.01} & {\scriptsize 95.97} & {\scriptsize 92.06}
\tabularnewline
\cmidrule{1-11}
{\scriptsize 20} & {\scriptsize 100.0} & {\scriptsize 100.0} & {\scriptsize 100.0} & {\scriptsize 99.95} & {\scriptsize 99.82} & {\scriptsize 99.60} & {\scriptsize 99.17} & {\scriptsize 98.01} & {\scriptsize 95.98} & {\scriptsize 92.06}
 \tabularnewline
\cmidrule{1-11}
{\scriptsize 30} & {\scriptsize  100.0}& {\scriptsize 100.0} & {\scriptsize 100.0} & {\scriptsize 99.98} & {\scriptsize 99.82} & {\scriptsize 99.63} & {\scriptsize  99.17} & {\scriptsize 98.0} & {\scriptsize 96.03} & {\scriptsize 92.06}
 \tabularnewline
\cmidrule{1-11}
{\scriptsize 40} & {\scriptsize  100.0} & {\scriptsize 100.0} & {\scriptsize 100.0} & {\scriptsize 99.98} & {\scriptsize 99.84} & {\scriptsize 99.67} & {\scriptsize 99.16} & {\scriptsize 98.01} & {\scriptsize 96.02} & {\scriptsize 92.06}
 \tabularnewline
\cmidrule{1-11}
{\scriptsize 50} & {\scriptsize  100.0} & {\scriptsize 100.0} & {\scriptsize 100.0} & {\scriptsize 99.98} & {\scriptsize 99.84} & {\scriptsize 99.65} & {\scriptsize 99.16} & {\scriptsize 98.05} & {\scriptsize 96.01} & {\scriptsize 92.06}
 \tabularnewline
\cmidrule{1-11}
{\scriptsize 60} & {\scriptsize  100.0} & {\scriptsize 100.0} & {\scriptsize 100.0} & {\scriptsize 99.98} & {\scriptsize 99.84} & {\scriptsize 99.65} & {\scriptsize 99.17} & {\scriptsize 98.05} & {\scriptsize 96.0} & {\scriptsize 92.06}
 \tabularnewline
\cmidrule{1-11}
{\scriptsize 70} & {\scriptsize  100.0} & {\scriptsize 100.0} & {\scriptsize 100.0} & {\scriptsize 99.98} & {\scriptsize 99.84} & {\scriptsize 99.65} & {\scriptsize 99.17} & {\scriptsize 98.075} & {\scriptsize 96.0} & {\scriptsize 92.06}
 \tabularnewline
\cmidrule{1-11}
{\scriptsize 80} & {\scriptsize  100.0} & {\scriptsize 100.0} & {\scriptsize 100.0} & {\scriptsize 99.98} & {\scriptsize 99.84} & {\scriptsize 99.65} & {\scriptsize 99.16} & {\scriptsize 98.06} & {\scriptsize 96.01} & {\scriptsize 92.06}
\tabularnewline
\cmidrule{1-11}
{\scriptsize 90} & {\scriptsize  100.0} & {\scriptsize 100.0} & {\scriptsize 100.0} & {\scriptsize 100.0} & {\scriptsize 99.84} & {\scriptsize 99.65} & {\scriptsize 99.14} & {\scriptsize 98.06} & {\scriptsize 96.01} & {\scriptsize 92.06}
\tabularnewline
\cmidrule{1-11}
{\scriptsize 100} & {\scriptsize  100.0} & {\scriptsize 100.0} & {\scriptsize 100.0} & {\scriptsize 100.0} & {\scriptsize 99.84} & {\scriptsize 99.65} & {\scriptsize 99.14} & {\scriptsize 98.09} & {\scriptsize 96.0} & {\scriptsize 92.06}
\tabularnewline
\bottomrule
\end{tabular}}{\scriptsize\par}
\par\end{centering}
\caption{Accuracy over the top-$k$ test samples, ranked by TRUST score, on the CIFAR-10 dataset with softmax temperature $T = 5.0$ and regularization parameter $\lambda = 0.001$ for the first training 100 epochs (in increments of 10). The values reported in the table correspond to those plotted in Fig \ref{fig:accs_fixtlambda}(a).}
\label{tab:top_percentage_acc_epoch100}
\end{table*}

\begin{table*}
\begin{centering}
{\footnotesize{}%
\begin{tabular}{@{}l@{}lc@{}}
\toprule 
{\footnotesize Dataset} & {\footnotesize Model} & {\footnotesize Accuracy (\%)}\tabularnewline
\midrule 
\multirow{5}{*}{{\footnotesize CIFAR-10}} & {\footnotesize SimpleNet} & {\footnotesize 74.0}\tabularnewline
 & {\footnotesize SimpleNet+} & {\footnotesize 84.0}\tabularnewline
 & {\footnotesize VGG11} & {\footnotesize 89.0}\tabularnewline
 & {\footnotesize PreactResNet18 (dropout)} & {\footnotesize 92.06 (89.0)}\tabularnewline
 & {\footnotesize ViT-Base} & {\footnotesize 71.61}\tabularnewline
\multirow{1}{*}{{\footnotesize CAMELYON-17}} & {\footnotesize PreactResNet18(dropout)} & {\footnotesize 83.52 (85.24)}\tabularnewline
\multirow{1}{*}{{\footnotesize TinyImagenet}} & {\footnotesize ResNet50 (dropout)} & 76.29 ({\footnotesize 81.71)}\tabularnewline
{\footnotesize Imagenet} & {\footnotesize ViT-Base}{\footnotesize\par}

{\footnotesize (pre-trained)} & {\footnotesize 85.62}\tabularnewline
\bottomrule
\end{tabular}}{\footnotesize\par}
\par\end{centering}
\caption{Accuracy of different trained models across various test datasets.}\label{tab:Performance}
\end{table*}

\section{Additional analysis on TRUST scores}

\begin{figure}
    \centering
\includegraphics[width=0.5\linewidth]{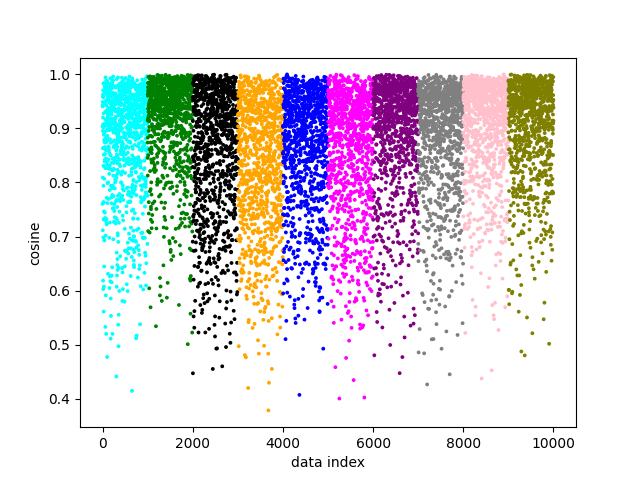}
    \caption{TRUST score scatter plots for CIFAR-10 test samples on the PreactResNet18 model, colored by class (Class 0 to 9: cyan to olive green).}
    \label{fig:scatter}
\end{figure}

\begin{figure*}
\centering
\subfloat[Class 0]{\includegraphics[scale=0.5]{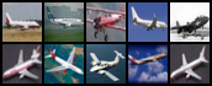}
}~~\subfloat[Class 1]{\includegraphics[scale=0.5]{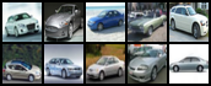}
}~~\subfloat[Class 2]{\includegraphics[scale=0.5]{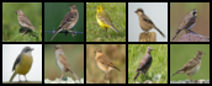}
}~~\subfloat[Class 3]{\includegraphics[scale=0.5]{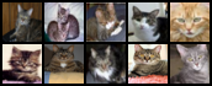}
}

~~\subfloat[Class 4]{\includegraphics[scale=0.5]{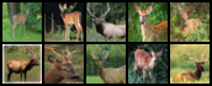}
}
\subfloat[Class 5]{\includegraphics[scale=0.5]{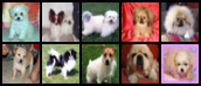}
}~~\subfloat[Class 6]{\includegraphics[scale=0.5]{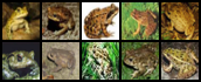}
}~~\subfloat[Class 7]{\includegraphics[scale=0.5]{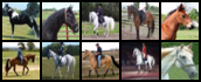}
}

~~\subfloat[Class 8]{\includegraphics[scale=0.5]{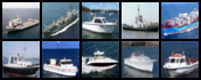}
}
~~\subfloat[Class 9]{\includegraphics[scale=0.5]{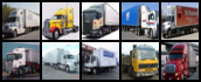}
}\caption{Mode samples with highest TRUST scores for each class (Class 0 to Class 9) in CIFAR-10.}\label{fig:Mode-(high-TRUST score)}
\end{figure*}

\begin{figure*}
\centering
\subfloat[Class 0]{\includegraphics[scale=0.5]{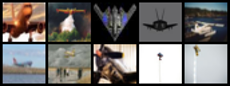}
}~~\subfloat[Class 1]{\includegraphics[scale=0.5]{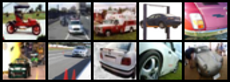}
}~~\subfloat[Class 2]{\includegraphics[scale=0.5]{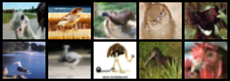}
}~~\subfloat[Class 3]{\includegraphics[scale=0.5]{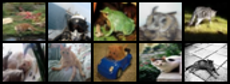}
}

~~\subfloat[Class 4]{\includegraphics[scale=0.5]{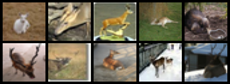}
}
\subfloat[Class 5]{\includegraphics[scale=0.5]{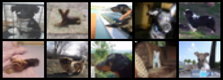}
}~~\subfloat[Class 6]{\includegraphics[scale=0.5]{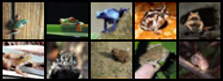}
}~~\subfloat[Class 7]{\includegraphics[scale=0.5]{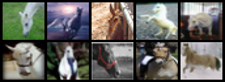}
}

~~\subfloat[Class 8]{\includegraphics[scale=0.5]{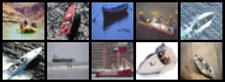}
}~~\subfloat[Class 9]{\includegraphics[scale=0.5]{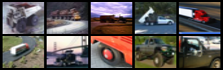}
}\caption{Mode samples with lowest
TRUST scores for each class (Class 0 to Class 9) in CIFAR-10.}\label{fig:Tail-or-rare}
\end{figure*}

\begin{figure*}[p]
\centering
\subfloat [\label{fig:b-1}]
{
\includegraphics[width=0.35\columnwidth]{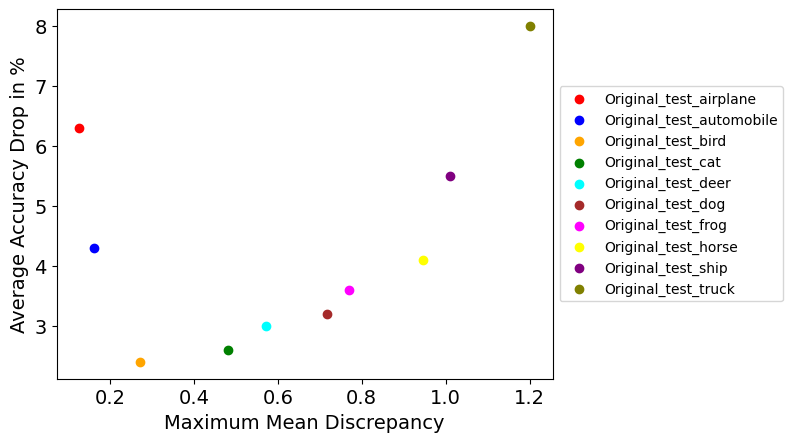}
}
\subfloat [\label{fig:b-1}]
{
\includegraphics[scale=0.3]{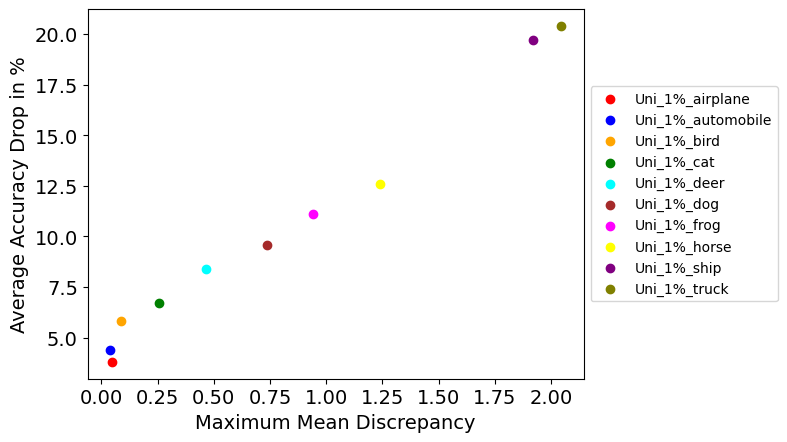}
}

\subfloat [\label{fig:b-1}]
{
\includegraphics[scale=0.3]{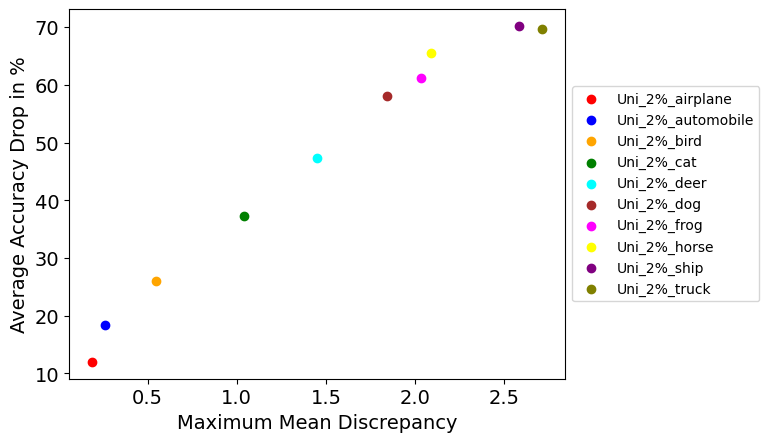}
}
\subfloat [\label{fig:b-1}]
{
\includegraphics[scale=0.3]{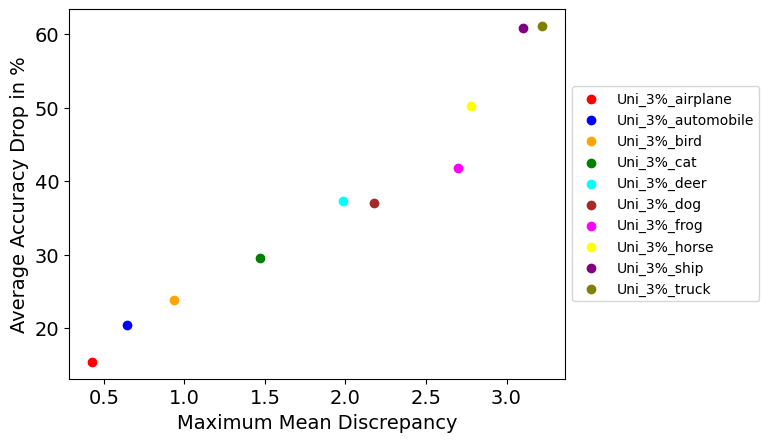}
}

\subfloat [\label{fig:b-1}]
{
\includegraphics[scale=0.3]{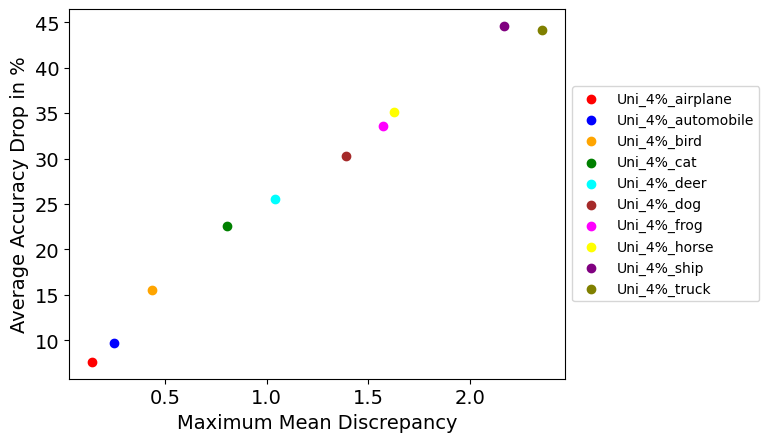}
}
\subfloat [\label{fig:b-1}]
{
\includegraphics[scale=0.3]{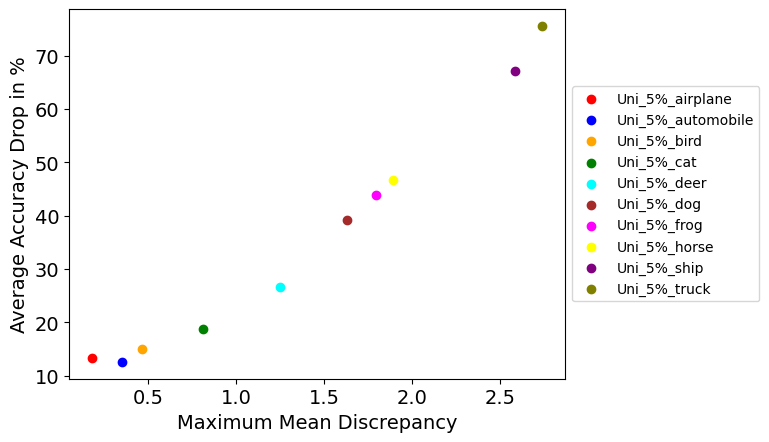}
}

\subfloat [\label{fig:b-1}]
{
\includegraphics[scale=0.3]{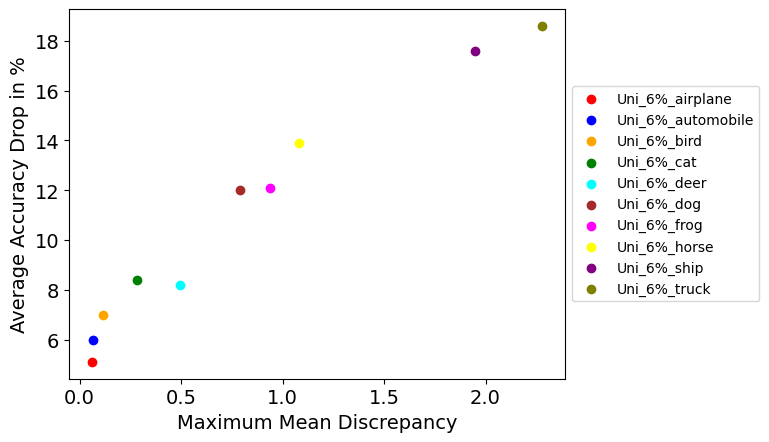}
}
\subfloat [\label{fig:b-1}]
{
\includegraphics[scale=0.3]{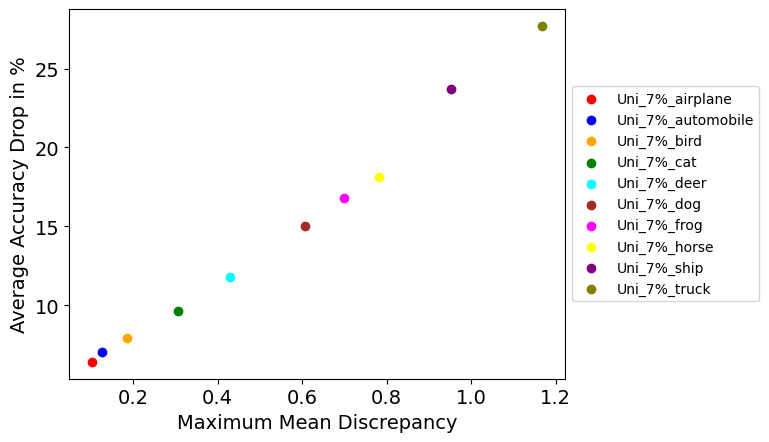}
}

\subfloat [\label{fig:b-1}]
{
\includegraphics[scale=0.3]{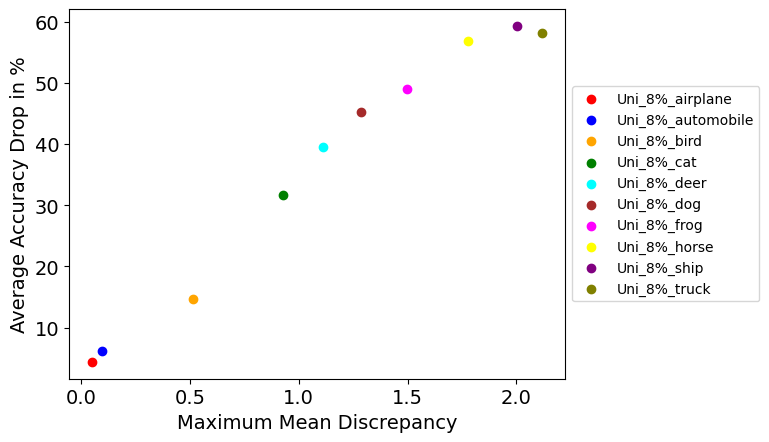}
}
\subfloat [\label{fig:b-1}]
{
\includegraphics[scale=0.3]{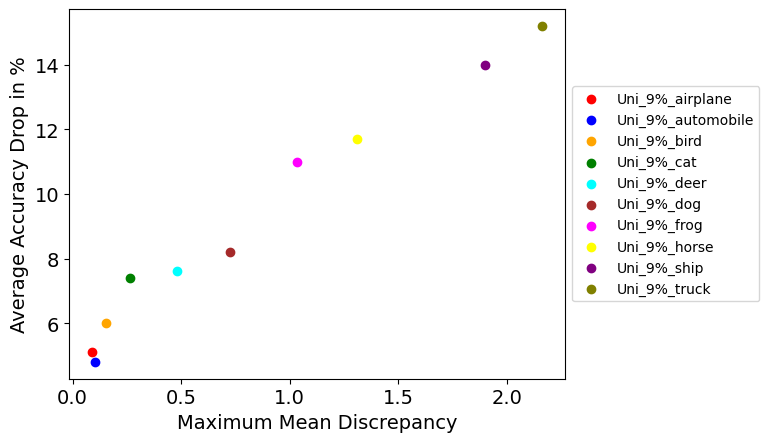}
}
\caption{Accuracy drop vs MMD for original test and uniform noise from 1\% to
\% 9\% for CIFAR-10 PreactResNet18 model.}
\label{fig:Accuracy-drop-vs_unif}
\end{figure*}

\begin{figure*}
    \centering
    \subfloat [\label{fig:b-1}]
{
\includegraphics[scale=0.3]{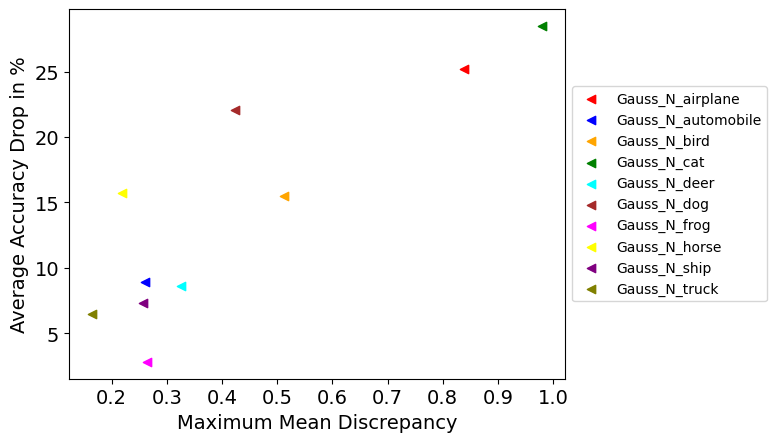}
}
\subfloat [\label{fig:b-1}]
{
\includegraphics[scale=0.3]{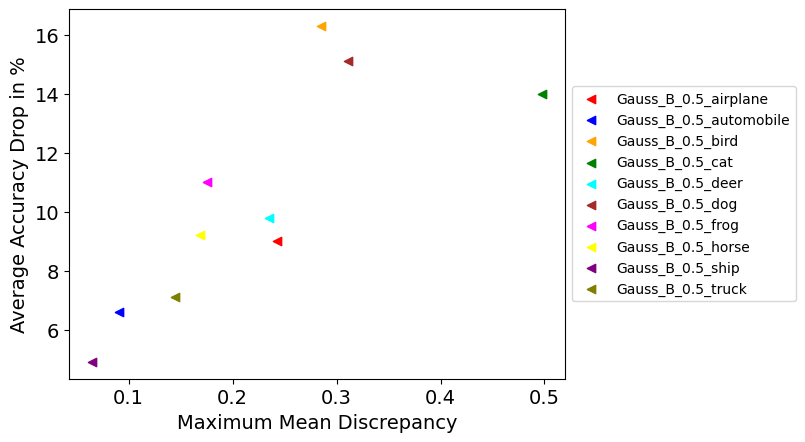}
}

\subfloat [\label{fig:b-1}]
{
\includegraphics[scale=0.3]{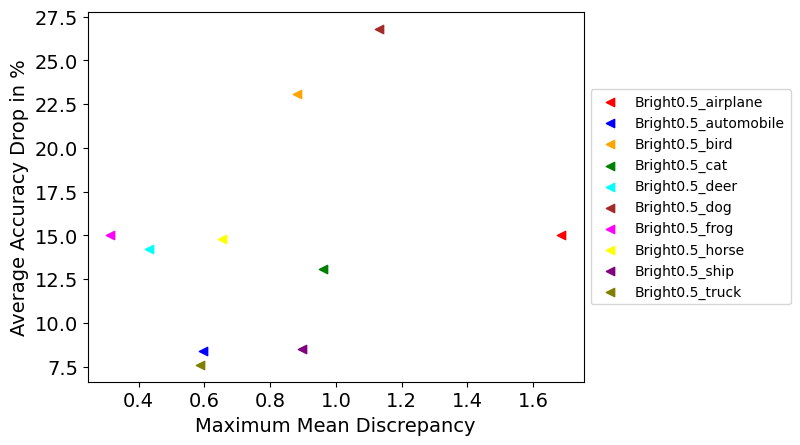}
}
\subfloat [\label{fig:b-1}]
{
\includegraphics[scale=0.3]{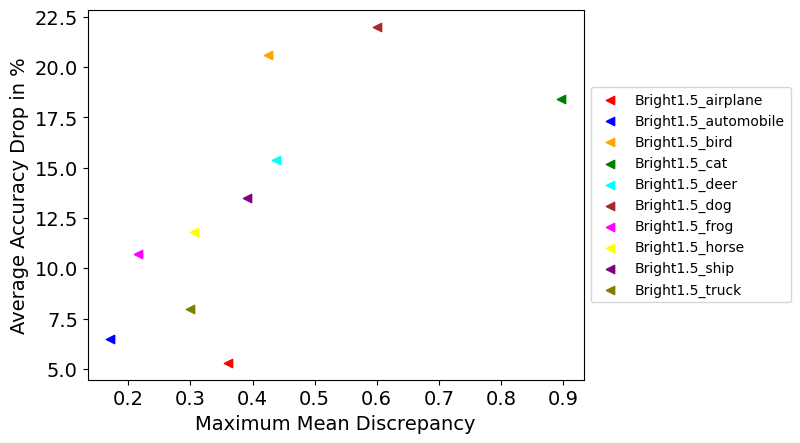}
}

\subfloat [\label{fig:b-1}]
{
\includegraphics[scale=0.3]{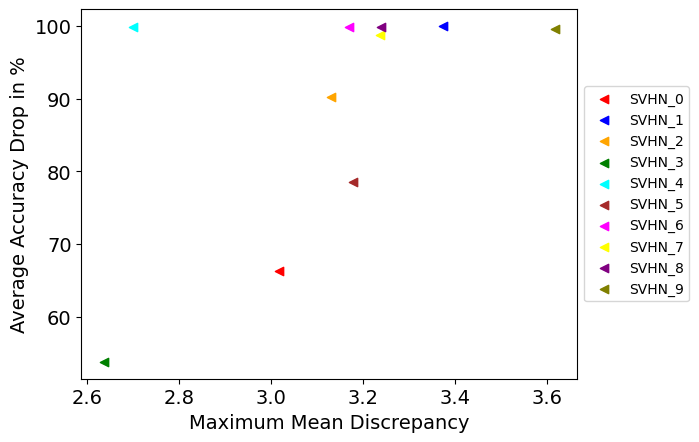}
}
\caption{Accuracy drop vs MMD Gaussian noise, Gaussian blur, different brightness and SVHN dataset for CIFAR-10 PreactResNet18 model.}
\label{fig:Accuracy-drop-vs_rest}
\end{figure*}

\begin{figure*}
\centering
\subfloat[Class 0]{\includegraphics[scale=0.2]{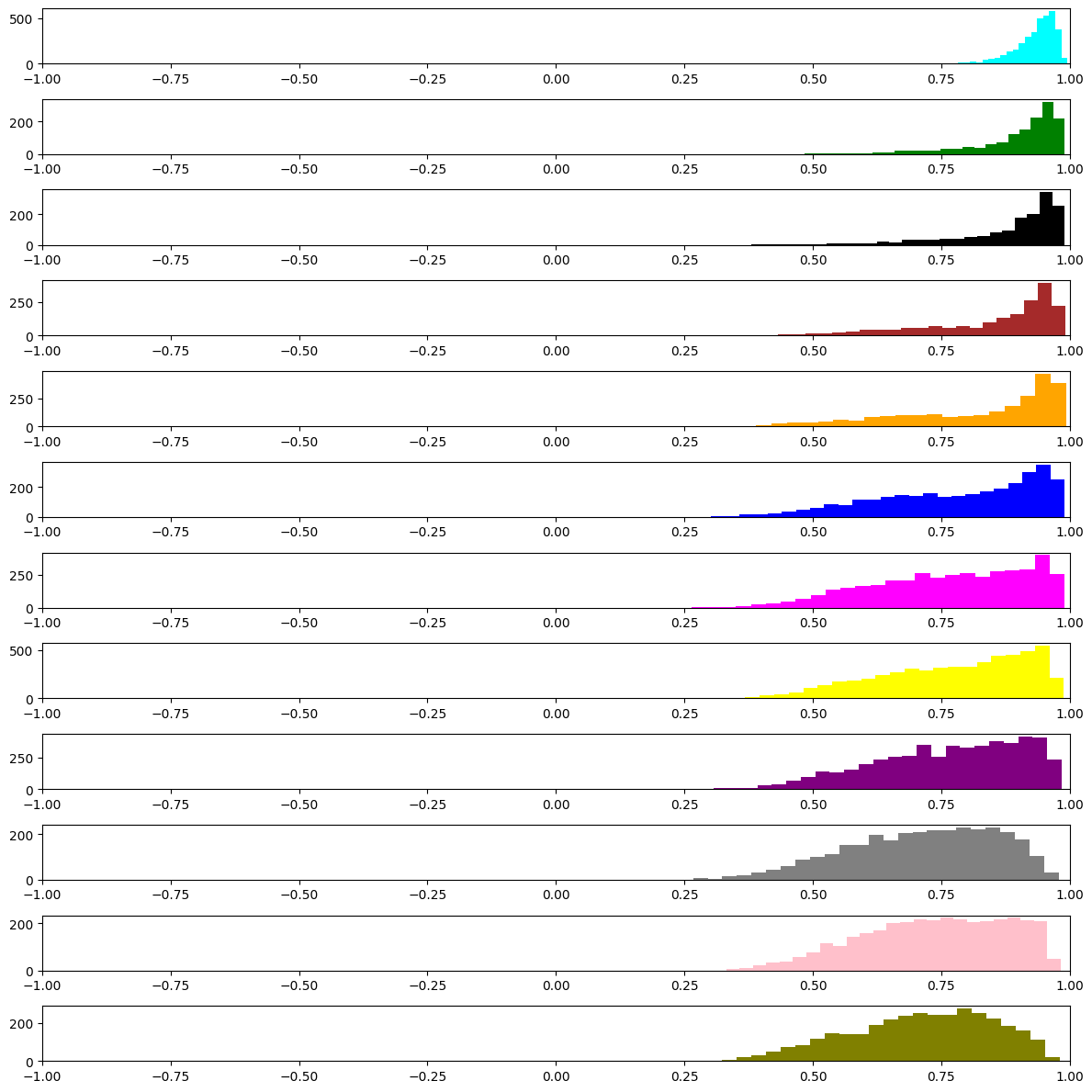}
}~~~~~~\subfloat[Class 1]{\includegraphics[scale=0.2]{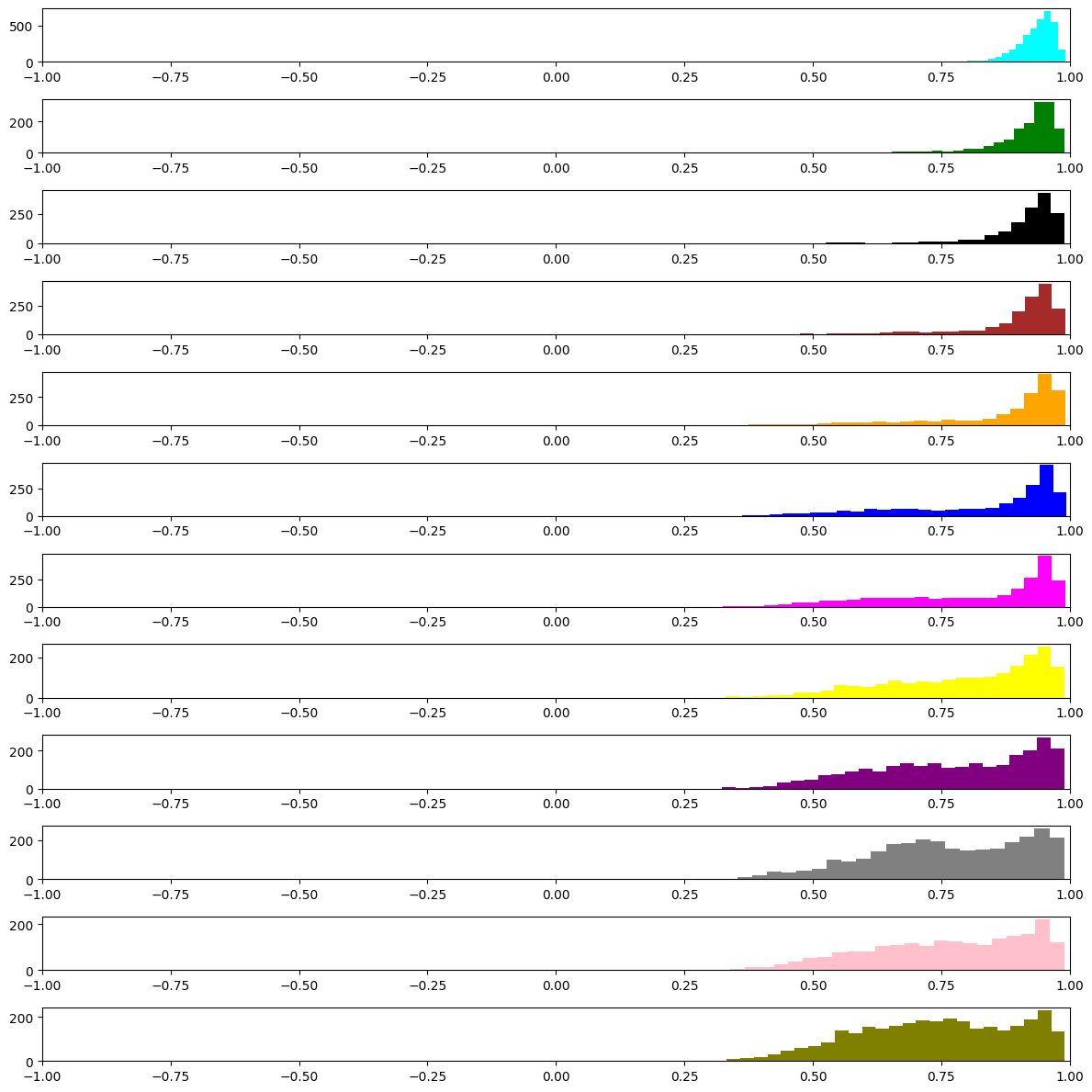}

}\\
\subfloat[Class 2]{\includegraphics[scale=0.2]{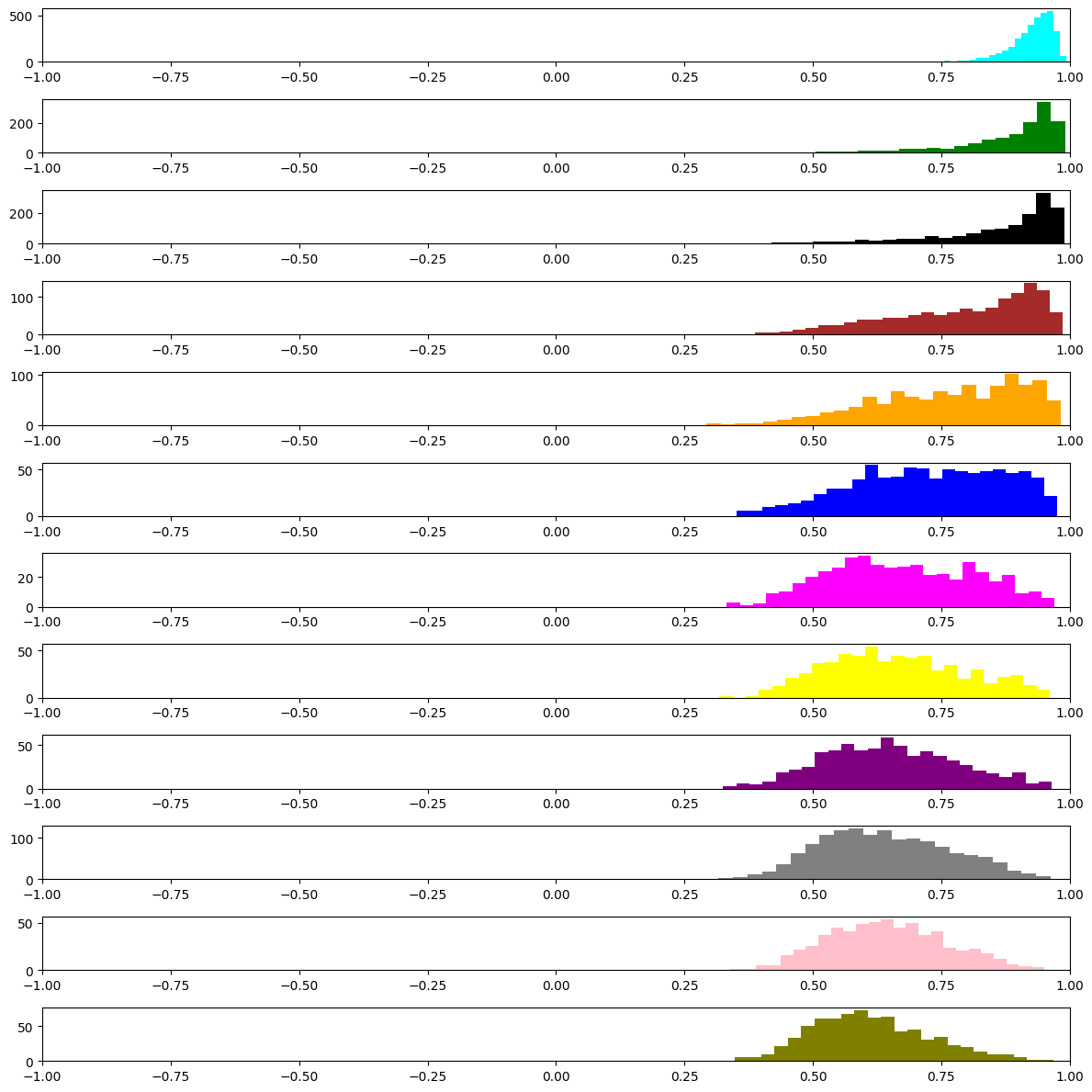}

}~~~~~~\subfloat[Class 3]{\includegraphics[scale=0.2]{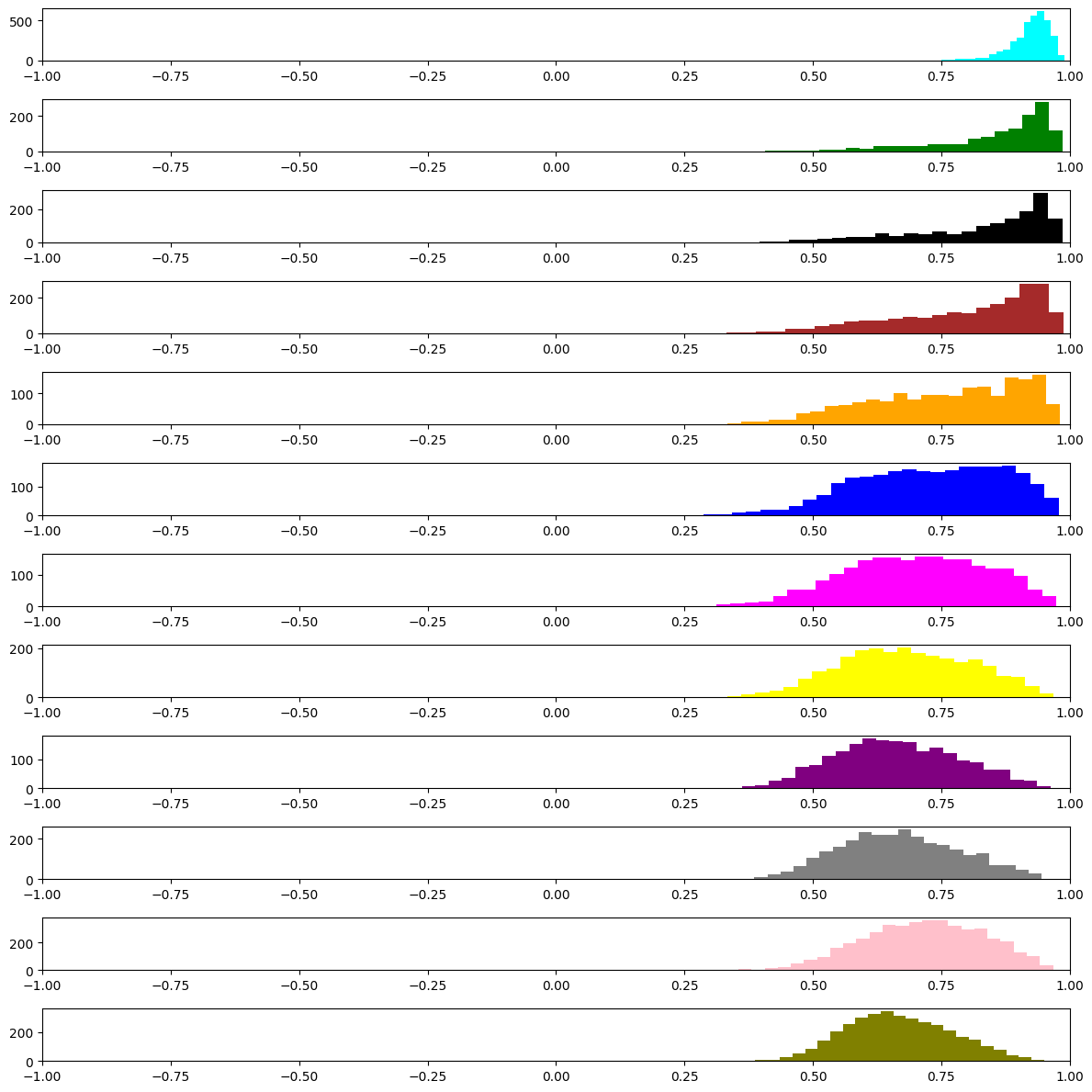}

}\\
\subfloat[Class 4]{\includegraphics[scale=0.2]{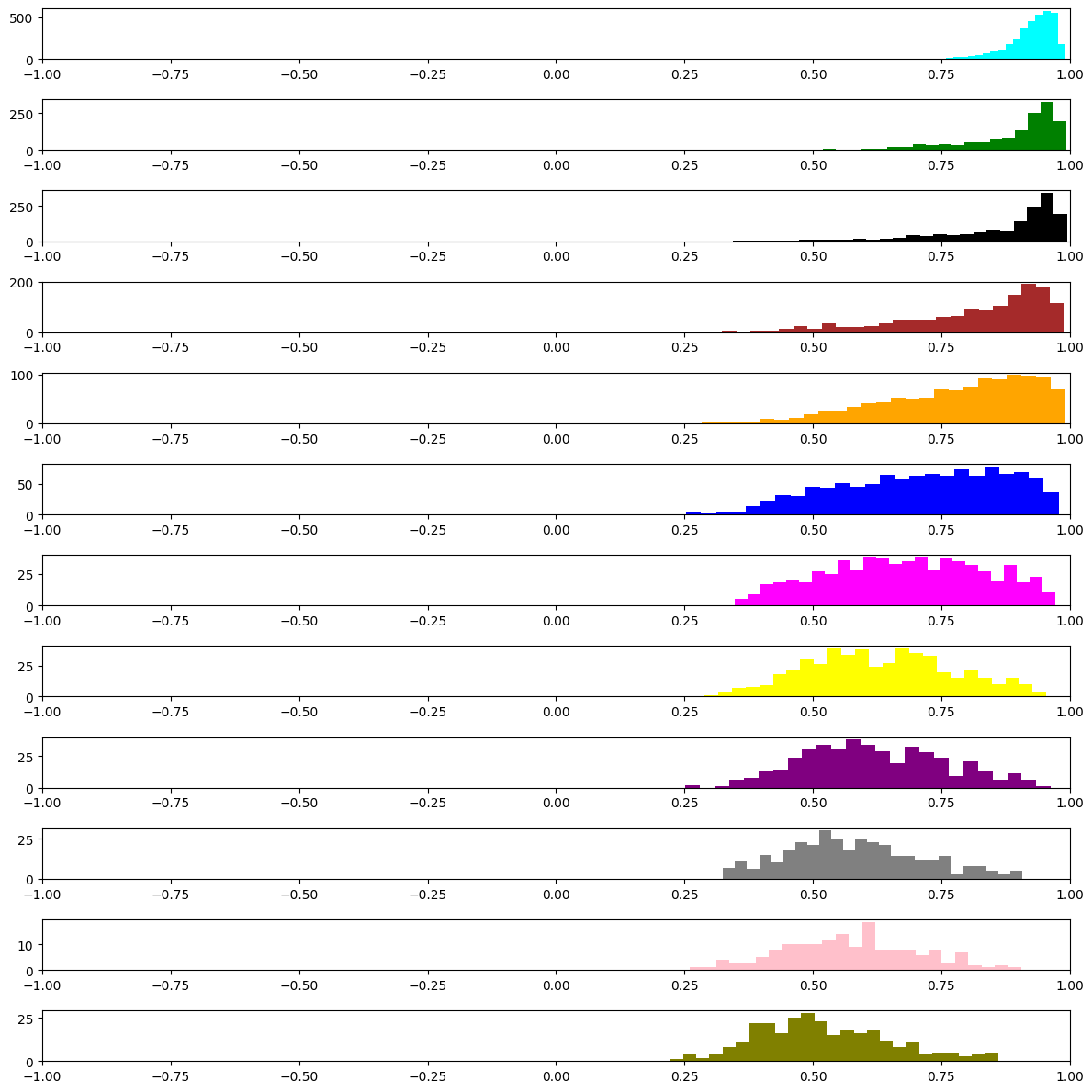}

}~~~~~~\subfloat[Class 5]{\includegraphics[scale=0.2]{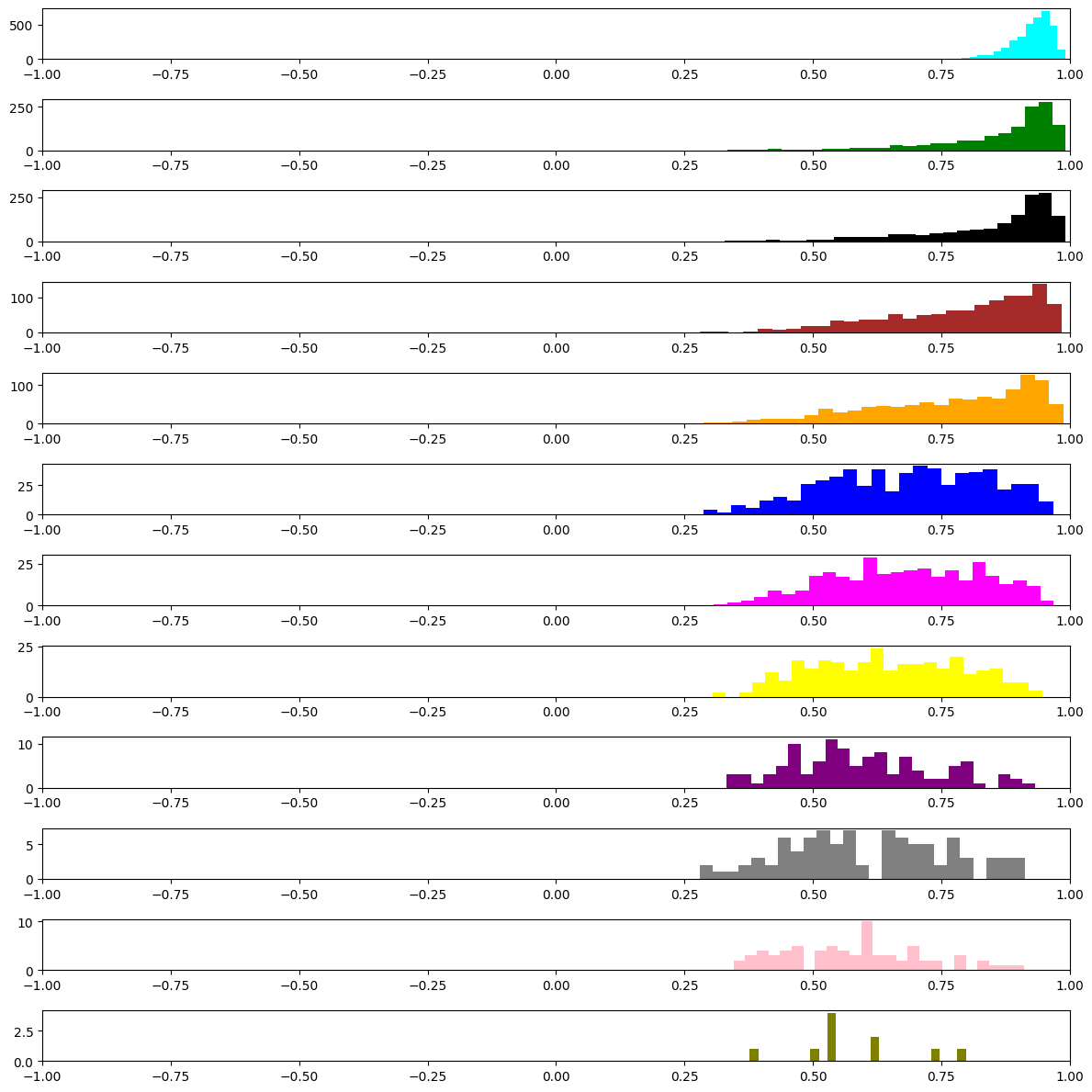}

}\\
\caption{Class-wise (Class 0 to Class 5) TRUST score distribution of CIFAR-10
train (in cyan), test (in green) and test with noises (uniform noises
from 0\% (in black) to 9\% (in olive) based on PreactResNet18 model.}\label{fig:trust_score}
\end{figure*}
\begin{figure*}
\centering
\subfloat[Class 6]{\includegraphics[scale=0.2]{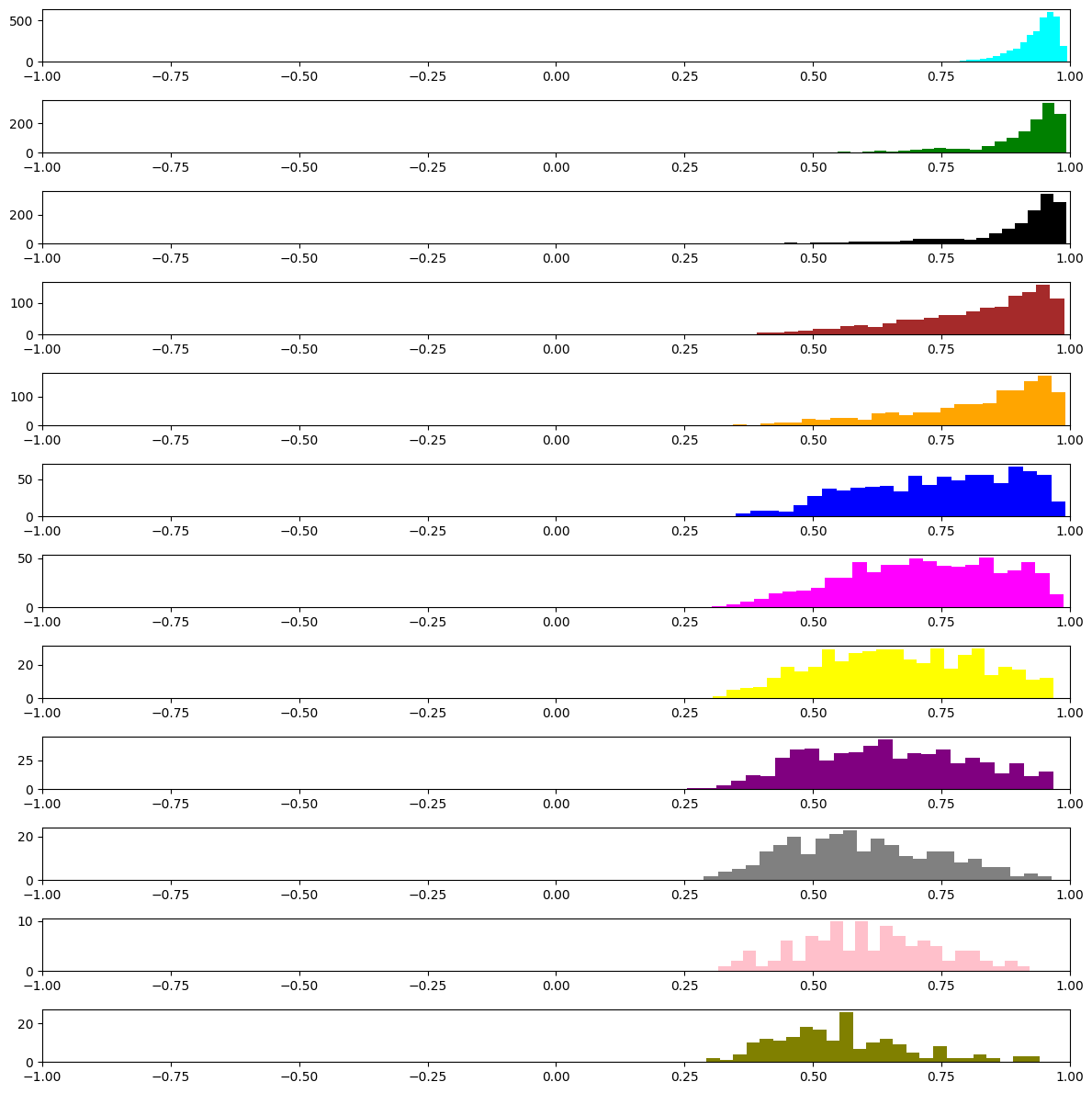}

}~~~~~~\subfloat[Class 7]{\includegraphics[scale=0.2]{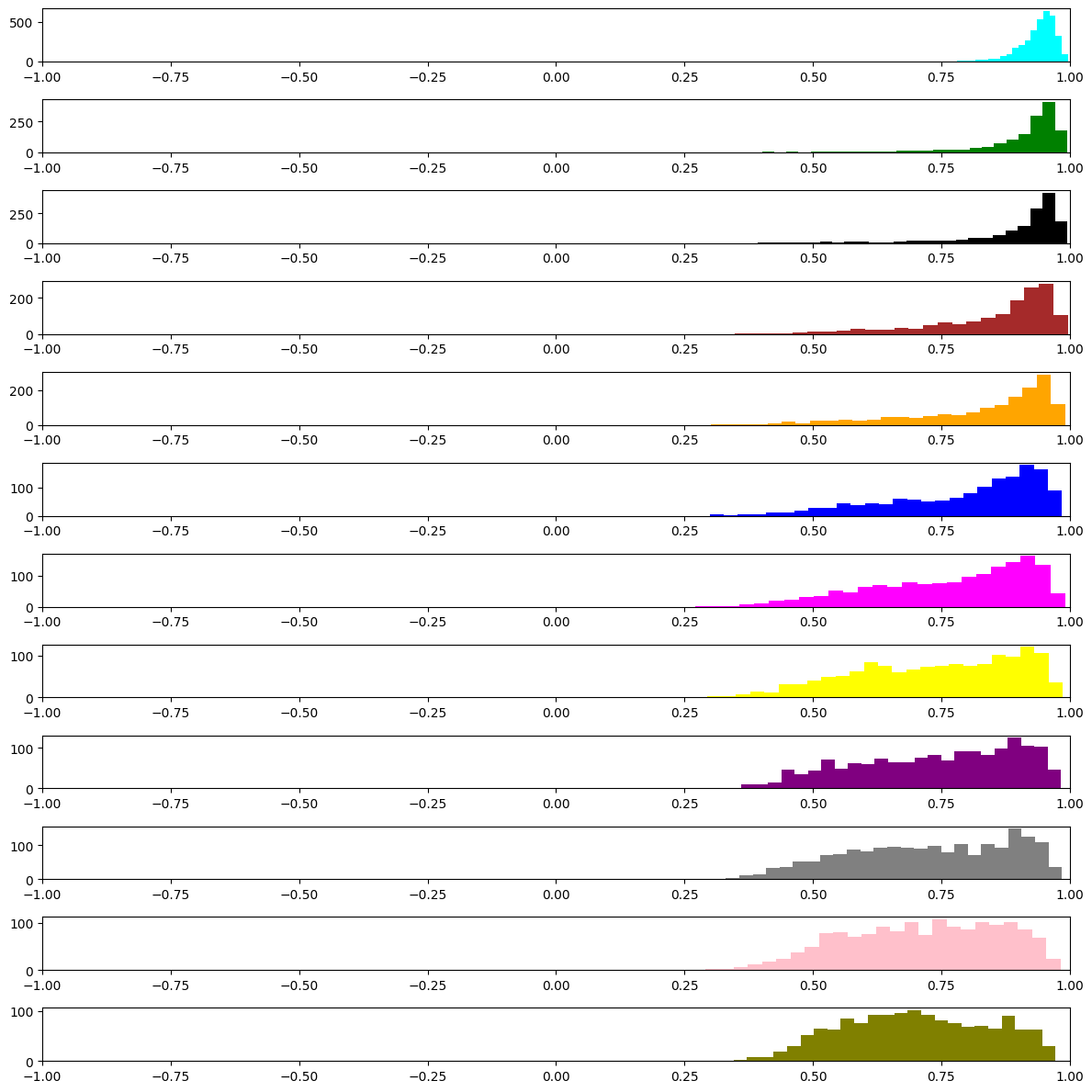}

}\\
\subfloat[Class 8]{\includegraphics[scale=0.2]{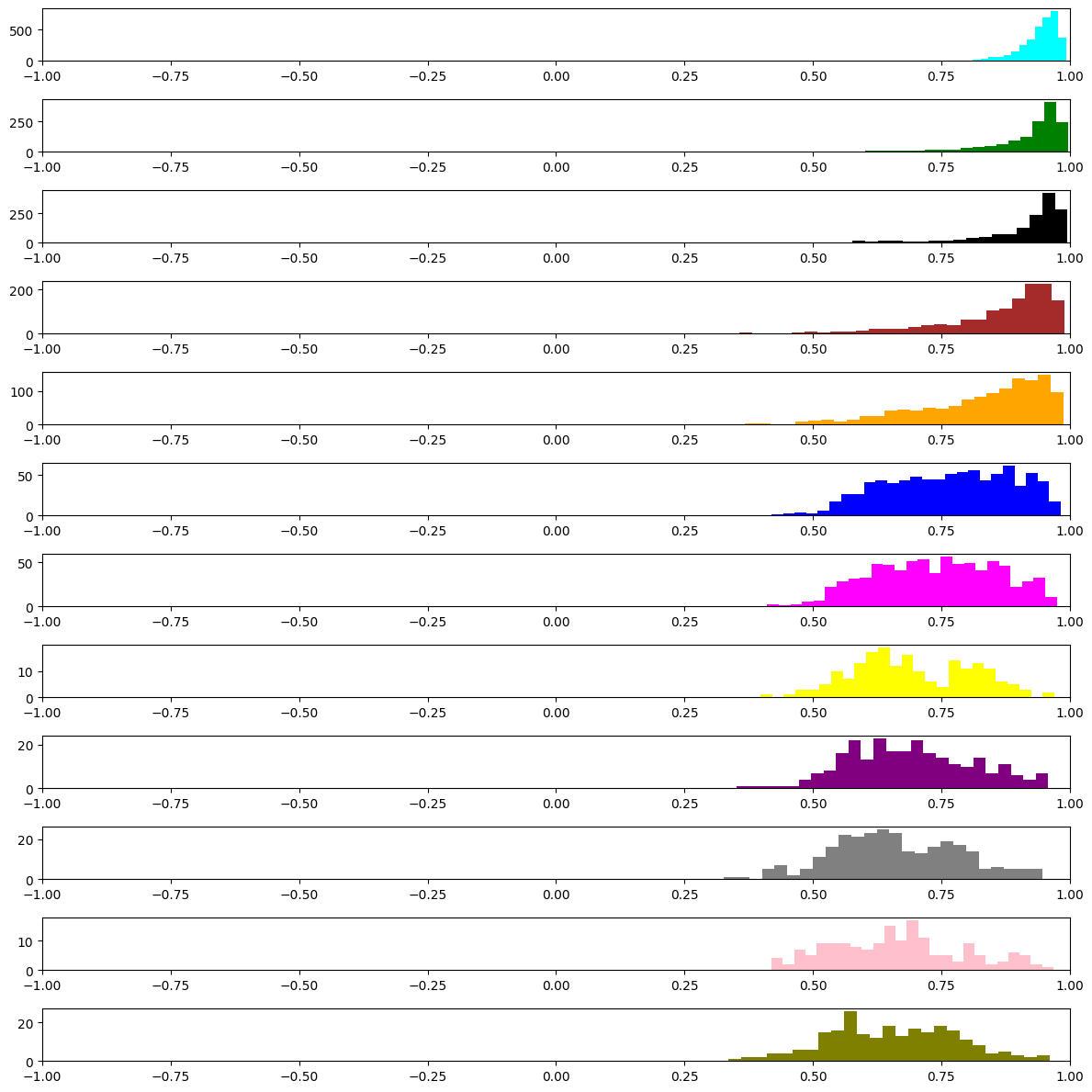}

}~~~~~~\subfloat[Class 9]{\includegraphics[scale=0.2]{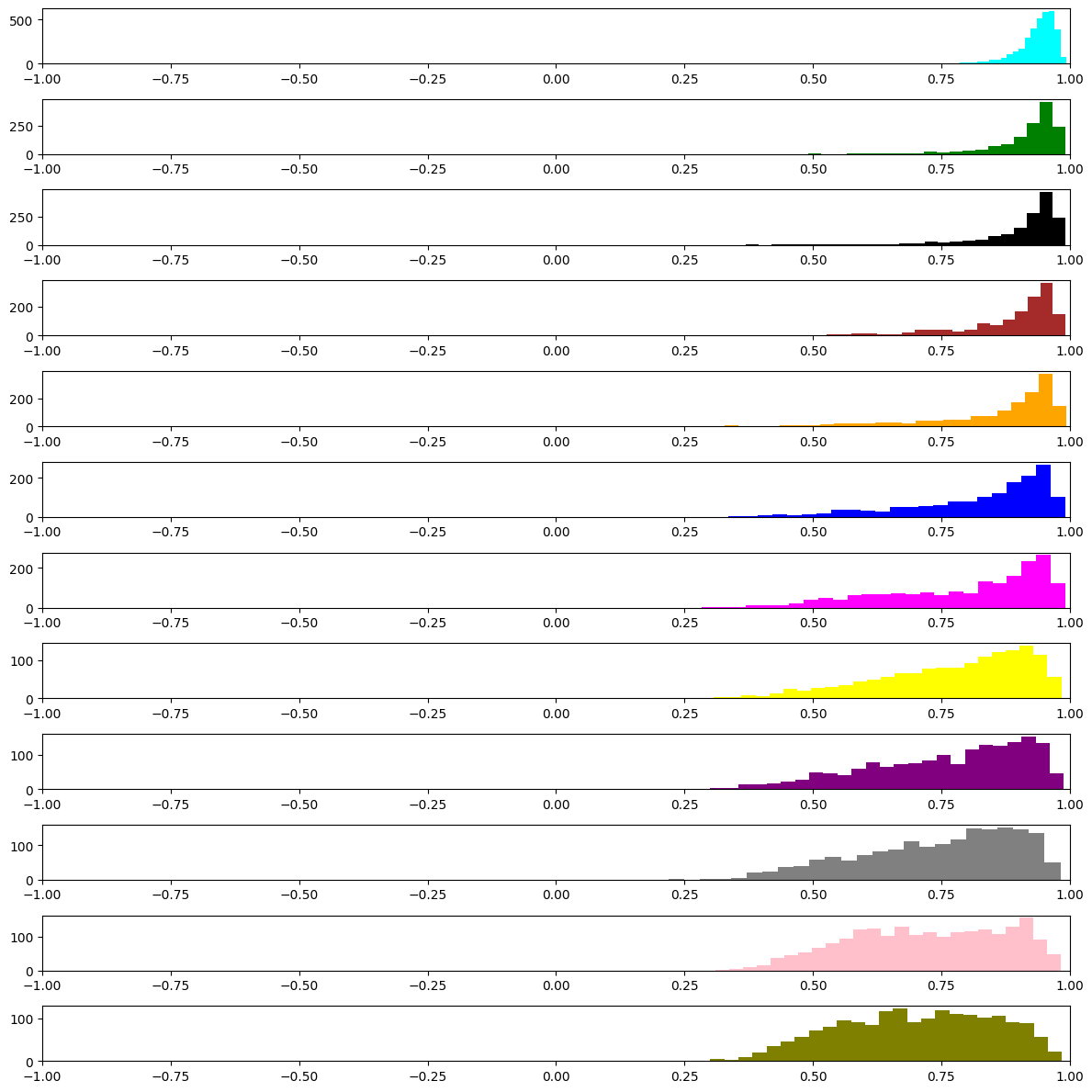}

}\caption{Class-wise (Class 6 to Class 9) TRUST score distribution of CIFAR-10
train (in cyan), test (in green) and test with noises (uniform noises
from 0\% (in black) to 9\% (in olive) based on PreactResNet18 model.}\label{fig:trust_score-1}
\end{figure*}

\subsection{TRUST score for dataset inspection}

The high score and low score samples from each class of the CIFAR-10 dataset are shown in Fig \ref{fig:Mode-(high-TRUST score)} and \ref{fig:Tail-or-rare}.
The high score samples represent some top modes from each class of
CIFAR-10 dataset and the low score contains samples that are rare
and wrongly labeled (label noise) from each class of CIFAR-10 dataset.
As shown in Fig \ref{fig:Mode-(high-TRUST score)} and \ref{fig:Tail-or-rare}
our proposed TRUST score efficiently picked up samples from each class.
The scatter plot of TRUST scores of CIFAR-10 test dataset is shown
in Fig \ref{fig:scatter}. In Fig \ref{fig:scatter},
lower TRUST scores represents samples in the tail region, and the higher
TRUST score represents samples in the main modes region of CIFAR-10
dataset.

\subsection{In understanding test data alignment}

The class-wise accuracy drop in percentage vs Maximum Mean Discrepancy
for original test, uniform noises from 1\% to 9\%, Gaussian noise,
Gaussian blur, and different brightness for CIFAR-10 PreactResNet18
model is shown in Fig \ref{fig:Accuracy-drop-vs_unif} and Fig \ref{fig:Accuracy-drop-vs_rest}. We also report
the SVHN values when we use it as OOD dataset for the CIFAR-10 PreactResNet18
model. The predicted class of a test set is used to compute its associated
perturbation.

The class-wise TRUST score distribution of CIFAR-10 train (in cyan),
test (in green) and test data with uniform noises (0\% (in black)
to 9\% (in olive)) for the CIFAR-10 PreactResNet18 model is shown
in Fig \ref{fig:trust_score} and Fig \ref{fig:trust_score-1}. It
is evident from the Fig \ref{fig:trust_score} and \ref{fig:trust_score-1}
that the distribution of TRUST score shifts farther away from the
original training data TRUST score distribution of CIFAR-10 dataset
with the addition of noises.
\subsection{AUSE plots of CIFAR-10 dataset}
The AUSE plots over all and per classes of CIFAR-10 test dataset for
CrossEntro+TRUST, LogitNorm, and  LogitNorm+TRUST
on PreactResNet18 model is shown in Fig \ref{fig:trust_score-2} and
Fig \ref{fig:trust_score-2-1}.
The AUSE value corresponding to Fig \ref{fig:trust_score-2} is
reported in Table 1 in the main paper.
\begin{figure*}
\centering
\subfloat[{\scriptsize CrossEntro+TRUST}]{\includegraphics[scale=0.22]{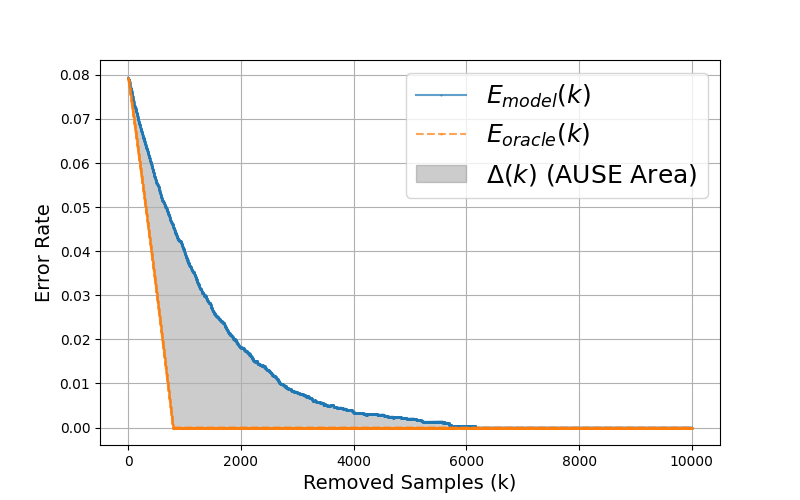}
}
\subfloat[LogitNorm]{\includegraphics[scale=0.22]{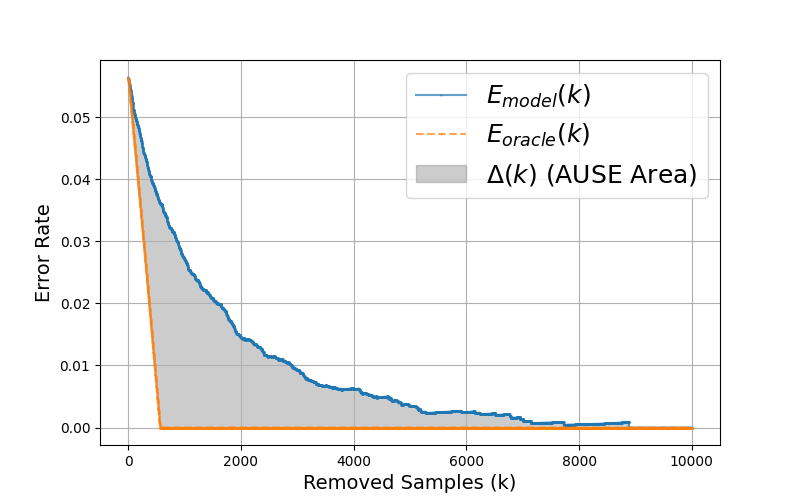}}~~~~~~\subfloat[{\scriptsize LogitNorm+TRUST}]{\includegraphics[scale=0.22]{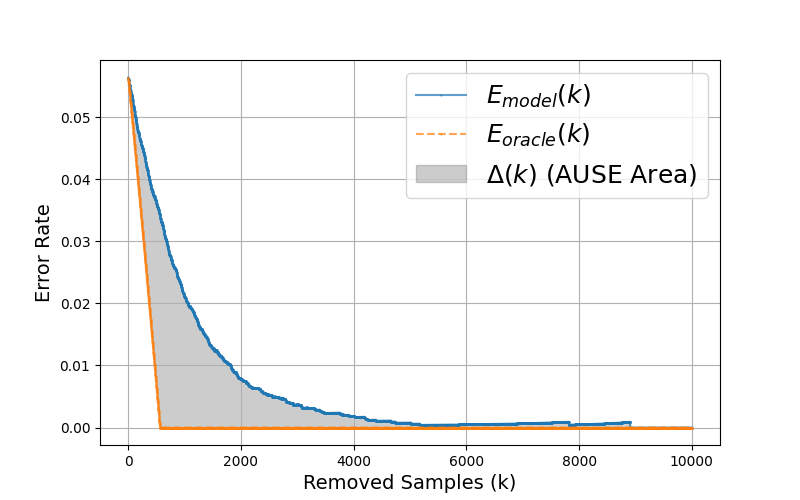}
}\caption{AUSE plots of CIFAR-10 test dataset on CrossEntro+TRUST,
LogitNorm, and LogitNorm+TRUST for PreactResNet18 model.}\label{fig:trust_score-2}
\end{figure*}
\begin{figure*}
\centering
\subfloat[{\scriptsize CrossEntro+TRUST}]{\includegraphics[scale=0.3]{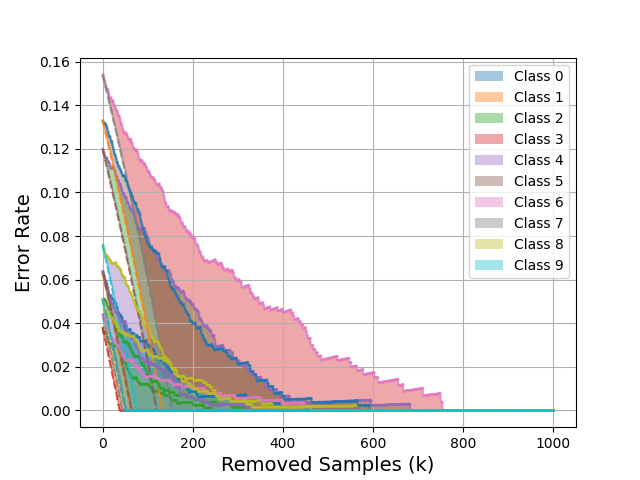}
}~\subfloat[LogitNorm]{\includegraphics[scale=0.3]{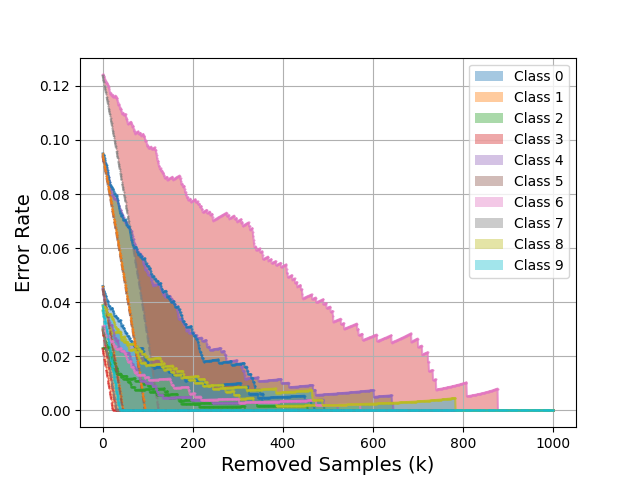}}~~~~~~\subfloat[{\scriptsize LogitNorm+TRUST}]{\includegraphics[scale=0.3]{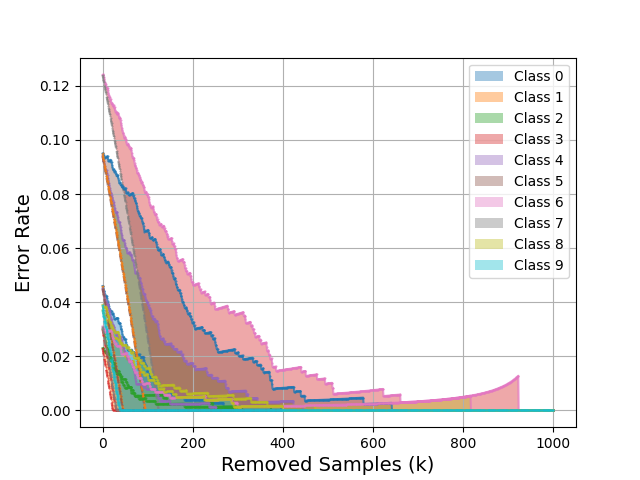}

}\caption{Class-wise AUSE plots of CIFAR-10 test dataset for CrossEntro+TRUST,
LogitNorm, and LogitNorm+TRUST for PreactResNet18 model.}\label{fig:trust_score-2-1}
\end{figure*}

%% file: main.bbl
\begin{thebibliography}{10}

\bibitem{alexey2020image}
Dosovitskiy Alexey.
\newblock An {I}mage is {W}orth 16x16 {W}ords: {T}ransformers for {I}mage {R}ecognition at {S}cale.
\newblock {\em arXiv preprint arXiv: 2010.11929}, 2020.

\bibitem{bakken2023ai}
Suzanne Bakken.
\newblock A{I} in {H}ealth: {K}eeping the {H}uman in the {L}oop, 2023.

\bibitem{camelyon}
Peter Bandi.
\newblock Camelyon17 dataset.

\bibitem{biderman2024emergent}
Stella Biderman, Usvsn Prashanth, Lintang Sutawika, Hailey Schoelkopf, Quentin Anthony, Shivanshu Purohit, and Edward Raff.
\newblock Emergent and {P}redictable {M}emorization in {L}arge {L}anguage {M}odels.
\newblock {\em Advances in Neural Information Processing Systems}, 2024.

\bibitem{carlini2022quantifying}
Nicholas Carlini, Daphne Ippolito, Matthew Jagielski, Katherine Lee, Florian Tramer, and Chiyuan Zhang.
\newblock Quantifying {M}emorization across {N}eural {L}anguage {M}odels.
\newblock {\em arXiv preprint arXiv:2202.07646}, 2022.

\bibitem{chua2023tackling}
Michelle Chua, Doyun Kim, Jongmun Choi, Nahyoung~G Lee, Vikram Deshpande, Joseph Schwab, Michael~H Lev, Ramon~G Gonzalez, Michael~S Gee, and Synho Do.
\newblock Tackling {P}rediction {U}ncertainty in {M}achine {L}earning for {H}ealthcare.
\newblock {\em Nature Biomedical Engineering}, 2023.

\bibitem{deng2023uncertainty}
Danruo Deng, Guangyong Chen, Yang Yu, Furui Liu, and Pheng-Ann Heng.
\newblock Uncertainty {E}stimation by {F}isher {I}nformation-{B}ased {E}vidential {D}eep {L}earning.
\newblock In {\em International Conference on Machine Learning}, 2023.

\bibitem{deng2009imagenet}
Jia Deng, Wei Dong, Richard Socher, Li-Jia Li, Kai Li, and Li~Fei-Fei.
\newblock Imagenet: {A} {L}arge-{S}cale {H}ierarchical {I}mage {D}atabase.
\newblock In {\em Proceedings of the IEEE/CVF International Conference on Computer Vision}, 2009.

\bibitem{wang2022vim}
Wang et~al.
\newblock Vi{M}: {O}ut-of-{D}istribution {W}ith {V}irtual-{L}ogit {M}atching.
\newblock In {\em Proceedings of the IEEE/CVF International Conference on Computer Vision}, 2022.

\bibitem{xia2022augmenting}
Xia et~al.
\newblock Augmenting {S}oftmax {I}nformation for {S}elective {C}lassification with {O}ut-of-{D}istribution {D}ata.
\newblock In {\em Asian Conference on Computer Vision}, 2022.

\bibitem{yoon2024uncertainty}
Yoon et~al.
\newblock Uncertainty {E}stimation by {D}ensity {A}ware {E}vidential {D}eep {L}earning.
\newblock {\em International {C}onference on {M}achine {L}earning}, 2024.

\bibitem{zhu2022rethinking}
Zhu et~al.
\newblock Rethinking {C}onfidence {C}alibration for {F}ailure {P}rediction.
\newblock In {\em European Conference on Computer Vision}, 2022.

\bibitem{zhu2023openmix}
Zhu et~al.
\newblock Open{M}ix: {E}xploring {O}utlier {S}amples for {M}isclassification {D}etection.
\newblock In {\em Proceedings of the IEEE/CVF International Conference on Computer Vision}, 2023.

\bibitem{zhu2024rcl}
Zhu et~al.
\newblock R{CL}: {R}eliable {C}ontinual {L}earning for {U}nified {F}ailure {D}etection.
\newblock In {\em Proceedings of the IEEE/CVF International Conference on Computer Vision}, 2024.

\bibitem{flosdorf2024skin}
Carolin Flosdorf, Justin Engelker, Igor Keller, and Nicolas Mohr.
\newblock Skin {C}ancer {D}etection utilizing {D}eep {L}earning: {C}lassification of {S}kin {L}esion {I}mages using a {V}ision {T}ransformer.
\newblock {\em arXiv preprint arXiv:2407.18554}, 2024.

\bibitem{gal2016dropout}
Yarin Gal and Zoubin Ghahramani.
\newblock Dropout as a {B}ayesian {A}pproximation: {R}epresenting {M}odel {U}ncertainty in {D}eep {L}earning.
\newblock In {\em International {C}onference on {M}achine {L}earning}, 2016.

\bibitem{gawlikowski2023survey}
Jakob Gawlikowski, Cedrique Rovile~Njieutcheu Tassi, Mohsin Ali, Jongseok Lee, Matthias Humt, Jianxiang Feng, Anna Kruspe, Rudolph Triebel, Peter Jung, Ribana Roscher, et~al.
\newblock A {S}urvey of {U}ncertainty in {D}eep {N}eural {N}etworks.
\newblock {\em Artificial Intelligence Review}, 2023.

\bibitem{gheisari2023deep}
Mehdi Gheisari, Fereshteh Ebrahimzadeh, Mohamadtaghi Rahimi, Mahdieh Moazzamigodarzi, Yang Liu, Pijush~Kanti Dutta~Pramanik, Mohammad~Ali Heravi, Abolfazl Mehbodniya, Mustafa Ghaderzadeh, Mohammad~Reza Feylizadeh, et~al.
\newblock Deep {L}earning: {A}pplications, {A}rchitectures, {M}odels, {T}ools, and {F}rameworks: {A} {C}omprehensive {S}urvey.
\newblock {\em CAAI Transactions on Intelligence Technology}, 2023.

\bibitem{Gray_11Entropy}
Robert~M Gray.
\newblock {\em Entropy and information theory}.
\newblock Springer Science \& Business Media, 2011.

\bibitem{heaton2017deep}
James~B Heaton, Nick~G Polson, and Jan~Hendrik Witte.
\newblock Deep {L}earning for {F}inance: {D}eep {P}ortfolios.
\newblock {\em Applied Stochastic Models in Business and Industry}, 2017.

\bibitem{hsu2020generalized}
Yen-Chang Hsu, Yilin Shen, Hongxia Jin, and Zsolt Kira.
\newblock Generalized {ODIN}: {D}etecting {O}ut-{O}f-{D}istribution {I}mage without {L}earning from {O}ut-{O}f-{D}istribution {D}ata.
\newblock In {\em Proceedings of the IEEE/CVF Conference on Computer Vision and Pattern Recognition}, 2020.

\bibitem{huang2023look}
Yuheng Huang, Jiayang Song, Zhijie Wang, Shengming Zhao, Huaming Chen, Felix Juefei-Xu, and Lei Ma.
\newblock Look {B}efore {Y}ou {L}eap: {A}n {E}xploratory {S}tudy of {U}ncertainty {M}easurement for {L}arge {L}anguage {M}odels.
\newblock {\em arXiv preprint arXiv:2307.10236}, 2023.

\bibitem{hullermeier2021aleatoric}
Eyke H{\"u}llermeier and Willem Waegeman.
\newblock Aleatoric and {E}pistemic {U}ncertainty in {M}achine {L}earning: {A}n {I}ntroduction to {C}oncepts and {M}ethods.
\newblock {\em Machine Learning}, 2021.

\bibitem{jiang2018trust}
Heinrich Jiang, Been Kim, Melody Guan, and Maya Gupta.
\newblock To trust or not to trust a classifier.
\newblock {\em Advances in neural information processing systems}, 31, 2018.

\bibitem{kendall2017uncertainties}
Alex Kendall and Yarin Gal.
\newblock What {U}ncertainties do we need in {B}ayesian {D}eep {L}earning for {C}omputer {V}ision?
\newblock {\em Advances in Neural Information Processing Systems}, 2017.

\bibitem{Kingma_Ba_14Adam}
Diederik~P Kingma and Jimmy Ba.
\newblock Adam: A {M}ethod for {S}tochastic {O}ptimization.
\newblock In {\em International {C}onference on {L}earning {R}epresentations}, 2015.

\bibitem{kingma2015variational}
Durk~P Kingma, Tim Salimans, and Max Welling.
\newblock Variational {D}ropout and the {L}ocal {R}eparameterization {T}rick.
\newblock {\em Advances in Neural Information Processing Systems}, 2015.

\bibitem{kononenko1989bayesian}
Igor Kononenko.
\newblock Bayesian {N}eural {N}etworks.
\newblock {\em Biological Cybernetics}, 1989.

\bibitem{Krizhevsky_etal_14Cifar}
Alex Krizhevsky, Vinod Nair, and Geoffrey Hinton.
\newblock The cifar-10 dataset.
\newblock {\em online: http://www. cs. toronto. edu/kriz/cifar. html}, 2014.

\bibitem{lakshminarayanan2017simple}
Balaji Lakshminarayanan, Alexander Pritzel, and Charles Blundell.
\newblock Simple and {S}calable {P}redictive {U}ncertainty {E}stimation using {D}eep {E}nsembles.
\newblock {\em Advances in Neural Information Processing Systems}, 2017.

\bibitem{liang2017enhancing}
Shiyu Liang, Yixuan Li, and Rayadurgam Srikant.
\newblock Enhancing the {R}eliability of {O}ut-{O}f-{D}istribution {I}mage {D}etection in {N}eural {N}etworks.
\newblock {\em arXiv preprint arXiv:1706.02690}, 2017.

\bibitem{lind2024uncertainty}
Simon~Kristoffersson Lind, Ziliang Xiong, Per-Erik Forss{\'e}n, and Volker Kr{\"u}ger.
\newblock Uncertainty {Q}uantification {M}etrics for {D}eep regression.
\newblock {\em Pattern Recognition Letters}, 186, 2024.

\bibitem{mosqueira2023human}
Eduardo Mosqueira-Rey, Elena Hern{\'a}ndez-Pereira, David Alonso-R{\'\i}os, Jos{\'e} Bobes-Bascar{\'a}n, and {\'A}ngel Fern{\'a}ndez-Leal.
\newblock Human-in-the-{L}oop {M}achine {L}earning: {A} {S}tate of the {A}rt.
\newblock {\em Artificial Intelligence Review}, 2023.

\bibitem{mukhoti2023deep}
Jishnu Mukhoti, Andreas Kirsch, Joost van Amersfoort, Philip~HS Torr, and Yarin Gal.
\newblock Deep {D}eterministic {U}ncertainty: {A} {N}ew {S}imple {B}aseline.
\newblock In {\em Proceedings of the IEEE/CVF Conference on Computer Vision and Pattern Recognition}, 2023.

\bibitem{netzer2011reading}
Yuval Netzer, Tao Wang, Adam Coates, Alessandro Bissacco, Baolin Wu, Andrew~Y Ng, et~al.
\newblock Reading {D}igits in {N}atural {I}mages with {U}nsupervised {F}eature {L}earning.
\newblock In {\em NeurIPS workshop on deep learning and unsupervised feature learning}, 2011.

\bibitem{ovadia2019can}
Yaniv Ovadia, Emily Fertig, Jie Ren, Zachary Nado, David Sculley, Sebastian Nowozin, Joshua Dillon, Balaji Lakshminarayanan, and Jasper Snoek.
\newblock Can {Y}ou {T}rust {Y}our {M}odel's {U}ncertainty? {E}valuating {P}redictive {U}ncertainty {U}nder {D}ataset {S}hift.
\newblock {\em Advances in neural information processing systems}, 2019.

\bibitem{pan2023towards}
Leyan Pan and Xinyuan Cao.
\newblock Towards understanding neural collapse: The effects of batch normalization and weight decay.
\newblock {\em arXiv preprint arXiv:2309.04644}, 2023.

\bibitem{pearce2021understanding}
Tim Pearce, Alexandra Brintrup, and Jun Zhu.
\newblock Understanding {S}oftmax {C}onfidence and {U}ncertainty.
\newblock {\em arXiv preprint arXiv:2106.04972}, 2021.

\bibitem{qu2022improving}
Haoxuan Qu, Yanchao Li, Lin~Geng Foo, Jason Kuen, Jiuxiang Gu, and Jun Liu.
\newblock Improving the {R}eliability for {C}onfidence {E}stimation.
\newblock In {\em European Conference on Computer Vision}, 2022.

\bibitem{ILSVRC15}
Olga Russakovsky, Jia Deng, Hao Su, Jonathan Krause, Sanjeev Satheesh, Sean Ma, Zhiheng Huang, Andrej Karpathy, Aditya Khosla, Michael Bernstein, Alexander~C. Berg, and Li~Fei-Fei.
\newblock {{I}mage{N}et {L}arge {S}cale {V}isual Recognition {C}hallenge}.
\newblock {\em International Journal of Computer Vision}, 2015.

\bibitem{shafaei2018less}
Alireza Shafaei, Mark Schmidt, and James~J Little.
\newblock A {L}ess {B}iased {E}valuation of {O}ut-{O}f-{D}istribution {S}ample {D}etectors.
\newblock {\em arXiv preprint arXiv:1809.04729}, 2018.

\bibitem{shannon1948mathematical}
Claude~Elwood Shannon.
\newblock A {M}athematical {T}heory of {C}ommunication.
\newblock {\em The Bell System Technical Journal}, 1948.

\bibitem{shergadwala2022human}
Murtuza~N Shergadwala, Himabindu Lakkaraju, and Krishnaram Kenthapadi.
\newblock A {H}uman-centric {P}erspective on {M}odel {M}onitoring.
\newblock In {\em AAAI Conference on Human Computation and Crowdsourcing}, 2022.

\bibitem{swiler2009epistemic}
Laura~P Swiler, Thomas~L Paez, and Randall~L Mayes.
\newblock Epistemic {U}ncertainty {Q}uantification {T}utorial.
\newblock In {\em Proceedings of International Modal Analysis Conference}, 2009.

\bibitem{van2024can}
Boris van Breugel, Nabeel Seedat, Fergus Imrie, and Mihaela van~der Schaar.
\newblock Can {Y}ou {R}ely on {Y}our {M}odel {E}valuation? {I}mproving {M}odel {E}valuation with {S}ynthetic {T}est {D}ata.
\newblock {\em Advances in Neural Information Processing Systems}, 2024.

\bibitem{wei2022mitigating}
Hongxin Wei, Renchunzi Xie, Hao Cheng, Lei Feng, Bo~An, and Yixuan Li.
\newblock Mitigating neural network overconfidence with logit normalization.
\newblock In {\em International {C}onference on {M}achine {L}earning}, 2022.

\bibitem{zhou2024novel}
Han Zhou, Jordy Van~Landeghem, Teodora Popordanoska, and Matthew~B Blaschko.
\newblock A {N}ovel {C}haracterization of the {P}opulation {A}rea {U}nder the {R}isk {C}overage {C}urve ({AURC}) and {R}ates of {F}inite {S}ample {E}stimators.
\newblock {\em arXiv preprint arXiv:2410.15361}, 2024.

\bibitem{zhou2023deep}
S~Kevin Zhou, Hayit Greenspan, and Dinggang Shen.
\newblock {\em Deep learning for medical image analysis}.
\newblock 2023.

\bibitem{zhu2023unsupervised}
Pengfei Zhu, Mengshi Qi, Xia Li, Weijian Li, and Huadong Ma.
\newblock Unsupervised {S}elf-driving {A}ttention {P}rediction via {U}ncertainty {M}ining and {K}nowledge {E}mbedding.
\newblock In {\em Proceedings of the IEEE/CVF International Conference on Computer Vision}, 2023.

\end{thebibliography}
